\documentclass[sigconf]{acmart}
\AtBeginDocument{%
  }

\usepackage{multirow}
\usepackage{adjustbox}
\usepackage{enumitem}
\usepackage{tablefootnote}
\usepackage{booktabs}
\usepackage{amsmath}       
\usepackage{algorithm}     
\usepackage{algorithmic}
\usepackage{enumitem}
\usepackage{colortbl} 

\newcommand{\squishlist}{
  \begin{list}{$\bullet$}
    { \setlength{\itemsep}{0pt}      \setlength{\parsep}{3pt}
      \setlength{\topsep}{3pt}       \setlength{\partopsep}{0pt}
      \setlength{\leftmargin}{1.5em} \setlength{\labelwidth}{1em}
      \setlength{\labelsep}{0.5em} } }
\newcommand{\squishend}{
  \end{list}  }

\usepackage{balance}

\copyrightyear{2026}
\acmYear{2026}
\setcopyright{cc}
\setcctype{by}
\acmConference[KDD 2026] {Proceedings of the 32nd ACM SIGKDD Conference on Knowledge Discovery and Data Mining V.1}{August 9--13, 2025}{Jeju Island, Republic of Korea.}
\acmBooktitle{Proceedings of the 32nd ACM SIGKDD Conference on Knowledge Discovery and Data Mining V.1 (KDD 2026), August 9--13, 2025, Jeju Island, Republic of Korea}
\acmISBN{979-8-4007-2258-5/2026/08}
\acmDOI{10.1145/3770854.3780315} 

\begin{document}

\title{LLM-based Few-Shot Early Rumor Detection with Imitation Agent}

\author{Fengzhu Zeng}
\authornote{Both authors contributed equally to this research.}
\email{fzzeng.2020@phdcs.smu.edu.sg}
\orcid{https://orcid.org/0009-0003-6623-5494}
\author{Qian Shao}
\authornotemark[1]
\email{qianshao.2020@phdcs.smu.edu.sg}
\orcid{https://orcid.org/0000-0003-4396-9754}
\affiliation{%
  \institution{Singapore Management University}
  \country{Singapore}
}

\author{Ling Cheng}
\affiliation{%
  \institution{Singapore Management University}
  \country{Singapore}}
  \orcid{https://orcid.org/0000-0002-2834-9728}
\email{lingcheng.2020@phdcs.smu.edu.sg}

\author{Wei Gao}
\affiliation{%
  \institution{Singapore Management University}
  \country{Singapore}
}
\orcid{https://orcid.org/0000-0003-2028-2407}
\email{weigao@smu.edu.sg}

\author{Shih-Fen Cheng}
\affiliation{%
 \institution{Singapore Management University}
 \country{Singapore}}
 \orcid{https://orcid.org/0000-0001-9398-7892}
\email{sfcheng@smu.edu.sg}

\author{Jing Ma}
\affiliation{%
  \institution{Hong Kong Baptist University}
  \country{Hong Kong}}
\orcid{https://orcid.org/0000-0002-7464-8331}
\email{majing@hkbu.edu.hk}

\author{Cheng Niu}
\affiliation{%
  \institution{Particle Media, Inc}
  \city{Mountain View}
  \state{California}
  \country{USA}}
\email{cheng.niu@newsbreak.com}

\renewcommand{\shortauthors}{Fengzhu Zeng et al.}

\begin{abstract}
Early Rumor Detection (EARD) aims to identify the earliest point at which a claim can be accurately classified based on a sequence of social media posts. This is especially challenging in data-scarce settings. While Large Language Models (LLMs) perform well in few-shot NLP tasks, they are not well-suited for time-series data and are computationally expensive for both training and inference. In this work, we propose a novel EARD framework that combines an autonomous agent and an LLM-based detection model, where the agent acts as a reliable decision-maker for \textit{early time point determination}, while the LLM serves as a powerful \textit{rumor detector}. This approach offers the first solution for few-shot EARD, necessitating only the training of a lightweight agent and allowing the LLM to remain training-free. Extensive experiments on four real-world datasets show our approach boosts performance across LLMs and surpasses existing EARD methods in accuracy and earliness.\footnote{Code is released at \url{https://github.com/znhy1024/Few-Shot-EARD}.}
\end{abstract}





\maketitle

\section{Introduction}
The spread of false rumors can have serious societal consequences, leading to widespread concerns. For example, in July 2024, a brutal knife attack occurred in Southport, UK. Within less than two hours, rumors began circulating on social media, falsely accusing the attacker of being a Muslim immigrant, sparking one of the most violent riots in UK history\footnote{\url{https://apnews.com/article/britain-riots-unrest-social-media-misinformation-attack-5824d3136675e10d6a25c9e17287c994}}. Despite government efforts to control the situation, the damage caused by such misinformation had already caused significant harm, highlighting the critical need for EArly Rumor Detection (EARD) approaches to ensure timely identification of rumorous claims. 

EARD aims to automatically determine an \emph{early} time point in a sequence of social posts related to an unverified claim, at which the prediction regarding whether that claim is a rumor shall be \emph{accurate}~\cite{zeng-gao-2022-early}. Specifically, the model continuously monitors a stream of posts, and at each time step it must assess whether the current moment is suitable for making an accurate prediction (i.e., classifying the claim as a rumor or non-rumor) based on the observed social posts. If it determines that such a decision can be made, the model outputs its prediction and terminates further observation; otherwise, it continues monitoring the stream.

The EARD task is challenging yet has received limited attention. Existing EARD methods typically rely on  recurrent neural networks (RNNs) for the rumor detection module, and the early time point prediction module often employs deep reinforcement learning~\cite{zhou-etal-2019-early}, fixed probability thresholds~\cite{song2019ced}, or the neural Hawkes Process~\cite{zeng-gao-2022-early}. These methods tightly couple the two modules for joint training in a recurrent manner, necessitating extensive annotated data for effective model training while lacking the flexibility to easily replace either module with a more powerful alternative.

New events on social media constantly emerge, while there could be rarely or even no annotated claims for them and each claim tends to have limited number of relevant posts available in the early stage of propagation. Meanwhile, the evolving nature of social conversations about rumorous events requires advanced text understanding capabilities for a detector.
Annotating training data under such dynamics is very costly~\citep{saakyan-etal-2021-covid}. Moreover, waiting for lots of posts to accumulate before making decision (e.g., for a intervention) can inadvertently allow rumors to spread even further. This dilemma highlights the critical need for effective EARD methods in data-limited scenarios, a challenge that existing approaches struggle to address.

Recently, Large Language Models (LLMs) have demonstrated promising capabilities in understanding social media text with limited data across various tasks~\citep{10.1145/3543873.3587605, zhang-etal-2024-sentiment, yang-etal-2024-reinforcement,Lan_Gao_Jin_Li_2024,10459901,roy-etal-2023-probing,liu2024largelanguagemodelsdetect}. However, LLMs face significant limitations when applied to the EARD task, as they are not inherently well-suited for handling time series data or representing sequential dependencies of such data~\cite{LLM-TS} while training LLMs for dealing with such data requires intensive resources. 
Additionally, in the context of EARD, the cost of inference can become prohibitively high because the model must generate predictions at every time step along the sequence of posts, with the context length increasing over time, further compounding computational overhead and latency. 
This raises an important question: Given the strong language understanding capabilities of LLMs, how can we efficiently and effectively utilize them for few-shot EARD?

In this work, we present a novel  framework for EARD in data-limited scenarios,  providing an agile and cost-effective solution. Our framework consists of two components: a lightweight autonomous neural agent and an LLM.
At a high level, the neural agent serves as a reliable decision-maker for determining the optimal early time point by continuously monitoring the stream of social media posts, tracking the evolution of such social conversations about a claim. Meanwhile, the LLM functions as a powerful rumor detector, analyzing the observed posts at the early time point determined by the neural agent. This framework not only leverages the LLM's strong few-shot capabilities in text understanding but also decouples the training of the early time point prediction module from the rumor detection module. As a result, only the lightweight agent requires training, allowing the LLM to remain training-free and reducing its inference frequency to a single prediction, which can remarkably lower the demand for intensive computation. 

We first model EARD as a Markov Decision Process (MDP), where the observed posts represent the state, the action involves deciding whether to trigger the LLM for rumor detection. The agent aims to maximize expected returns, reflecting its ability to make early and accurate detection. However, the expected return depends on an effective reward function that can be highly challenging to design in a complex, real-world setting~\cite{ng2000algorithms}. In EARD, reward for determining the optimal early time point is especially difficult to quantify, relying on the sufficiency of observed information and the capability of the detector. 
To address this, we adopt imitation learning (IL), which enables the agent to learn optimal policy from a few expert trajectories, ultimately imitating the expert's decision-making process. IL has been successfully applied to many sequential decision-making problems~\citep{kuefler2017imitating,hussein2017imitation,ingimundardottir2018discovering,fang2019survey}, and in our case, it allows us to bypass the need for explicitly defining complex reward functions. Specifically, we design three types of strategies to curate expert trajectories based on the LLM's predictions at each time step along the sequence: the first two focus on state-action pairs that lead to early, stable, and correct predictions, and the third captures the state-action pairs where the LLM fails to make a correct prediction even when considering all available posts.
Consequently, we obtain three sets of state-action pairs, and the distribution of each set is referred to as the \textit{occupancy measure}, which accounts for both expert's policy and sequence dynamics. Our goal is to minimize the distance between the occupancy measure of the agent and those of the first two experts, while maximizing the divergence from the third expert. This is to ensure that the agent learns an optimal policy for better generalization by effectively imitating correct prediction patterns and avoiding the behaviors represented by incorrect predictions.  

Our main contributions are four-fold: 
    \begin{itemize}[leftmargin=*]
       \item We propose the first few-shot EARD method by introducing a framework that learns a lightweight autonomous neural agent for optimal time determination together with a training-free LLM-based rumor detector, offering a cost-effective solution. 
        \item We model EARD as an MDP, design three types of expert trajectories based on the predictions of LLMs over sequences of posts, and train the lightweight agent to effectively imitate correct prediction patterns while learning to avoid behaviors leading to incorrect predictions. 
        \item We theoretically prove that our proposed method can lead to an optimal policy for early, stable, and accurate detection.
        \item Extensive experiments on four real-world rumor detection datasets demonstrate the effectiveness of our proposed method across various LLMs, including Mistral~\cite{Mistral}, Llama 3~\cite{llama3}, ChatGPT~\cite{ChatGPT}. Our approach also outperforms existing EARD methods in terms of detection accuracy and earliness metrics.  
    \end{itemize}

\section{Related Work}
Many studies in rumor detection assert that their models can be generally adopted for early detection by simply feeding them data observed up to a predetermined checkpoint that could be either a specific time point or number of observed posts~\citep{ma2016detecting,yu2017convolutional,ma2017detect,ma-etal-2018-rumor,guo2018rumor,bian2020rumor,Lin_Yi_Ma_Jiang_Luo_Shi_Liu_2023}. Yet, such methods do not consider how to determine an optimal detection point. Typically, the checkpoints are chosen manually and coarse-grained. It is infeasible to determine an optimal early detection point among these checkpoints, since the decision requires referring to data points or model outputs beyond the current decision point (e.g., for the stability of prediction), thus delaying detection. Some other methods claim to be designed for early detection by integrating temporal factor into their classifier~\citep{zhao2015enquiring, nguyen2017early, wu2017gleaning, xia-etal-2020-state}, but they still lack mechanisms to enforce earliness and cannot automatically establish an optimal detection point. 

There are several \textit{full-shot} EARD methods that is capable of automatically deciding whether to stop or continue at a checkpoint. A method called ERD~\cite{zhou-etal-2019-early} used deep reinforcement learning with a manually designed reward function to encourage the model to focus on early time intervals during training.
~\citet{song2019ced} proposed another EARD method named Credible Detection Point (CED), which used a fixed probability threshold to determine whether to halt the detection process based on the credibility of the current prediction. Later, HEARD~\cite{zeng-gao-2022-early} highlighted the instability and low confidence of these models, noticing that they fail to account for the uncertainty of future predictions during training. HEARD addressed this by employing the neural Hawkes Process~\cite{mei-etal-2017-The} to construct a detection stability distribution based on a sequence of prior and current predictions, which allows the model to automatically determine a detection point at which predictions in future time steps remain unchanged. However, such cautiousness can lead to delayed detection as the decision is pending for achieving stabilization, while the noisy posts during rumor propagation can easily disturb the stabilization process.
Also, these methods are full-shot, requiring extensive annotated data for effective training, which makes them unsuitable for data-limited scenarios. In contrast, our few-shot method is specifically designed for the situation where only a few labeled instances are available, thereby aligning with real-world data limitations especially in EARD.

\section{Problem Definition} \label{problem definition}
Let $\mathcal{C} = \{C\}$ denote a set of instances, where each $C=\{M, y\}$ consists of a set of relevant posts in chronological order about an event $M = \{(\mathbf{m}_0, t_0), \ldots, (\mathbf{m}_i, t_i), \ldots, (\mathbf{m}_{|M|}, t_{|M|}) \}$ and the ground-truth label $y \in \{0, 1\}$ indicating $C$ is a rumor if $y = 1$ or a non-rumor otherwise. $|M|$ is the number of relevant posts in $M$ and each tuple $(\mathbf{m}_i, t_i) \in M$ includes the text content $\mathbf{m}_i$ and the timestamp $t_i$ of the $i$-th post ($t_{i} \leq t_{i+1}$). 
We define the EARD task as automatically determining the \textit{earliest} time point  \( \hat{t} \in \{ t_0,..,t_{|M|} \} \), such that for a given event the prediction \( \hat{y} \in \{ \text{non-rumor}, \text{rumor} \} \) at \( \hat{t} \) is \textit{accurate}. In the few-shot setting, we randomly sample $K$ instances from the training set. We do not assume the availability of a development set, as this aligns to a more realistic scenario with limited data~\cite{lee-etal-2021-towards,zeng-gao-2023-prompt,zeng-gao-2024-justilm}.

\begin{figure*}
    \includegraphics[width=1.\linewidth]{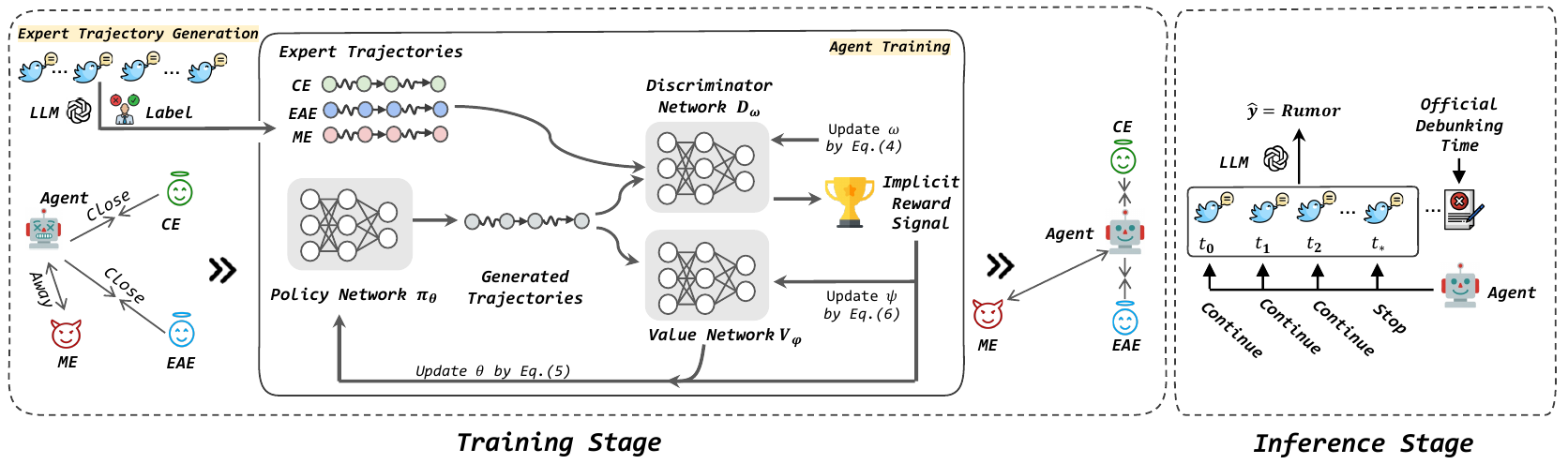}
\caption{llustration of the proposed EARD framework. During training, three types of expert trajectories are generated based on the LLM's predictions and the label. We utilize imitation learning to train the agent to find an optimal policy $\pi$ that aligns itself with CE and EAE while moves away from ME. At Inference, the trained agent automatically determines an early time point for the LLM to perform rumor detection. CE: Conservative Expert; EAE: Early-Action Expert; ME: Misleading Expert.
}
    \label{fig: pipeline}
\end{figure*}

\section{Methodology}
Figure \ref{fig: pipeline} shows the proposed framework, comprising a lightweight agent for automatic early time point determination and an LLM for rumor detection. Only the lightweight agent requires training, while we keep the LLM detector training-free. This framework decouples early time point prediction from rumor detection, allowing the LLM to utilize its powerful text understanding capabilities without heavy computational demands. The goal of the EARD task is to achieve timely and accurate detection, relying on both the observed posts and the LLM's detection abilities. Thus, the agent aims to learn the complex dynamics of LLM's prediction behaviors in response to different posts over time. By doing so, the framework leverages the complementary strengths of proactive timing control and reactive content evaluation to improve early rumor detection performance.

\subsection{MDP Formulation for EARD}\label{MDP Formulation}
We model the EARD task as a Markov Decision Process (MDP), represented by the tuple $\mathcal{M} = (\mathcal{S}, \mathcal{A}, \mathcal{P},\mathcal{R})$. Here, $\mathcal{S}$ denotes the state space, $\mathcal{A}$ is the action space, $\mathcal{P}$ represents the transition function, and $\mathcal{R}$ is the reward function. Each state $s_t\in\mathcal{S}$ consists of all observed posts up to the current step $t$ along with the previous action. 
Each action $ a_t \in \mathcal{A}$ is binary, comprising either \textit{continue} (to keep observing the stream of posts, denoted as $0$) or \textit{stop} (to terminate observation and trigger the LLM to make a prediction, denoted as $1$). 
The transition function $\mathcal{P}$ is deterministic: given the current state $s_t=(\mathbf{m}_{0:t},a_{t-1})$ and the action $a_t$, the environment transitions to the next state $s_{t+1}=(\mathbf{m}_{0:t+1},a_t)$. The reward function $r(s, a) \in \mathcal{R}$ reflects the immediate reward for state-action pairs, and ${\bar{\gamma}}\in (0, 1)$ is the discount factor. Let $\Pi$ represent the set of all stationary stochastic policies that select actions from $\mathcal{A}$ based on states in $\mathcal{S}$. 
We operate within a ${\bar{\gamma}}$-discounted infinite horizon setting, using the expectation with respect to a policy $\pi \in \Pi$ to denote the expected return over trajectories: $\mathbb{E}_{\pi}[r(s, a)] \triangleq \mathbb{E}_{\pi}\left[\sum_{t=0}^{\infty} {\bar{\gamma}}^t r(s_t, a_t)\right],$ where $s_0 \sim p_0$, $a_t \sim \pi(\cdot \mid s_t)$, and $s_{t+1} \sim \mathcal{P}(\cdot \mid s_t, a_t)$ for $t \geq 0$. Here, $p_0$ is the initial state distribution. Notably, the process forcibly terminates when the state contains the last post in the timeline.

\subsection{Imitation Learning} 
In the MDP formulation for EARD, the agent aims to find an optimal policy $\pi^{*}$ that maximizes the expected return, which is fundamentally influenced by the reward function measuring how early and accurate detection can be achieved. Reinforcement Learning (RL) methods typically rely on well-defined, observable reward function to guide the agent’s actions. However, creating an effective reward function can be extremely challenging in many real-world scenarios~\cite{ng2000algorithms}, including EARD. 

In EARD, determining the early time point relies on both the observed information and the detector's capability, both of which are difficult to quantify. Manually specified reward functions may be sparse or misaligned with true task objectives, making it difficult for the agent to learn an optimal policy and potentially leading to unintended behaviors that might cause delayed predictions~\cite{zeng-gao-2022-early}. Learning a reward function can be costly, requiring large labeled datasets and substantial computational resources to train predefined neural networks~\cite{ho2016generative,RLHF-cons}.  

To address these challenges, we utilize imitation learning (IL) to train an autonomous agent, offering a practical alternative that has been successfully applied in many sequential decision-making problems~\citep{kuefler2017imitating,hussein2017imitation,ingimundardottir2018discovering,fang2019survey}.
IL is particularly advantageous in environments where designing an explicit reward function is difficult or infeasible~\cite{ng2000algorithms}, as it enables the agent to learn from expert demonstrations without the need for an explicitly predefined reward function. IL can achieve good performance with fewer data samples~\cite{bain1995framework}, making it more efficient compared to RL that often requires extensive exploration. Moreover, by directly mimicking the decision-making process of the expert, IL typically converges faster to effective policies than RL which relies on trial-and-error exploration to discover optimal actions~\cite{ross2011reduction}. In the IL setting, we are provided with a limited number of expert trajectories from experts. These trajectories consist of state-action pairs that illustrate the actions taken by the expert in response to various states in the environment. Our goal is to learn a policy for the agent that imitates the expert's behavior.

\vspace{-1ex}
\subsubsection{Trajectories Generation} \label{Trajectories Generation} 
As the LLM continuously observes the stream of posts, it makes predictions $\hat{y}_i$ (i.e., rumor or non-rumor) at each time step $t_i$. Once the observation is terminated, a  sequence of predictions $\hat{y}_0,\hat{y}_1,..., \hat{y}_{|M|}$ of length $|M|$ is generated. By comparing these predictions with the ground-truth label $y$, we can construct expert trajectories for the post sequence, ultimately identifying an early time point for outputting a prediction. 

We generate expert trajectories based on LLM's prediction sequence, allowing the expert's policy to implicitly consider both the observed posts and the LLM's detection capability. This process is flexible enough to incorporate the expert's preferences by adjusting generation strategies considering the task. Hence, we design multiple type of experts for early time point determination, focusing on earliness, stability and accuracy. 
Specifically, let the early time point selected by the expert $E$ for an instance with sequence length $|M|$ be denoted as $t_{i^*}$ (where $0 \leq i^*\leq |M|$). An expert trajectory for this expert is represented as $\tau_{E} = \{(s_0,a_0),...,(s_{i^*},a_{i^*}),...(s_{|M|},a_{|M|})\}$, where $s_j \in \mathcal{S}$ and $a_j \in \mathcal{A}$; $a_j=0$ for $0\leq j < i^*$ and $a_j=1$ for $i^* \leq j\leq |M|$. Next, we introduce three types of experts.

\textit{Conservative Expert (CE):} 
Inspired by \citet{zeng-gao-2022-early}, this expert determines the early time point $t_{i^*}$, at which the prediction for a claim is accurate and remains unchanged over time. This means there are no reversals in the prediction after $t_{i^*}$; formally, $t_{i^*}=\text{min}\{t_{k}\;|\; \hat{y}_{k}=\hat{y}_{j}=y, \text{for all}\; j,\; k <j\leq |M|, k \in \{0,\ldots,|M|\}\}$. The stop action is taken when predictions are stable without further changes. However, this cautious approach can delay detection, as it is always likely that the model might be misled to change its prediction by irrelevant or distracting posts during rumor propagation. 

\textit{Early-Action Expert (EAE):} 
This expert builds on the CE's strategy but takes a more proactive approach. It selects the first time point, at which the rumor detector's prediction aligns with the ground truth, regardless of subsequent predictions. Formally, $t_{i^*}=\text{min}\{t_{k}\;|\; \hat{y}_{k}=y, k \in \{0,\ldots,|M|\}\}$. Intuitively, this expert allows for the earliest possible prediction but might sacrifice  stability in some extent. 

\textit{Misleading Expert (ME):}
This expert addresses scenarios where the LLM generates incorrect predictions when all available posts in the entire sequence have been considered.
This expert captures behaviors that lead to such incorrect decisions, with $t_{i^*}=\text{min}\{t_{k}\;|\; \hat{y}_{k} \neq y \; \text{and} \; \hat{y}_{k} = \hat{y}_{j}, \; \text{for all}\; j,\; k <j\leq |M|, \; k \in \{0,\ldots,|M|\}\}$. By incorporating this expert, the agent learns how to avoid misleading behaviors. 

Basically, our agent is trained to not only effectively imitate the \textit{CE} for early and stable predictions, but also adopt a more proactive manner to early time point determination through \textit{EAE}, and meanwhile avoid misleading behaviors by learning from the \textit{ME}. This synthesis of decision-making strategies from different experts enables the agent to balance diverse experts' behaviors and develop a more comprehensive and robust policy. As a result, the agent can flexibly adjust its strategy to  ensure accuracy and timeliness.

\subsubsection{Objective Function} \label{Objective Function}

\begin{figure}
    \vspace{-0ex}
    \includegraphics[width=0.48\textwidth]{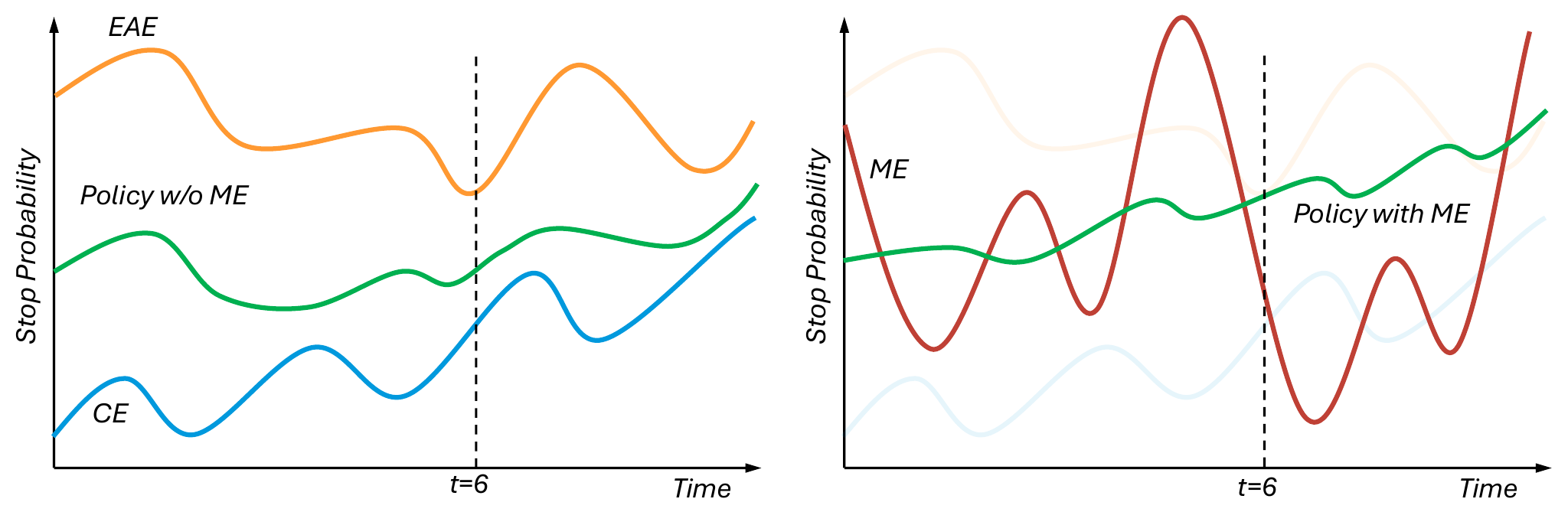}
\caption{An intuitive illustration of probability of stop action for CE, EAE, and ME over time. When the agent aligns itself with EAE and CE, it may hover around $0.5$ at some time points (e.g., $t=6$), leading to random choices. By including ME, the agent is encouraged to move away from the state-action distribution of ME, thereby increasing the stop probability.}
    \label{fig:MTE example}
\end{figure}
With our trajectory generation strategies, we obtain three sets of expert trajectories $\Phi_{E}=\{\tau_{E}\sim \pi_{E}\}$ for training, where $E = \{c, e, m\}$ correspond to \textit{CE}, \textit{EAE}, and \textit{ME}, respectively. These trajectories consist of state-action pairs that demonstrate the expert policy $\pi_{E}$ in the given environment. The agent is expected to learn an optimal policy that synthesizes the policies of different experts, i.e., learning from desirable expert behaviors while suppressing those from misleading ones. To achieve this, we employ generative adversarial imitation learning~\cite{ho2016generative}, a widely used imitation learning algorithm that frames the imitation learning problem as matching the occupancy measures between the expert and agent policies. The occupancy measure reflects the distribution of state-action pairs and accounts for both the policy and the environment dynamics. Specifically, the occupancy measure of $\pi$ is denoted as $\rho_{\pi} \in \Gamma$,  defined as $\rho_{\pi} (s, a) = \pi(a|s) \sum_{t=0}^\infty \bar{\gamma}^t \mathcal{P}(s_t = s|\pi)$~\cite{puterman2014markov}, where $\Gamma$ is the set of occupancy measures associated with all stationary stochastic policies $\Pi$.

Since there is an one-to-one correspondence between $\Pi$ and $\Gamma$~\cite{syed2008apprenticeship}, learning a policy $\pi$ can be framed as a matching problem between the occupancy measure of the agent policy $\rho_\pi$ and that of the expert policy $\rho_{\pi_E}$~\cite{ho2016generative}. Therefore, in our case, the agent learns a policy $\pi$ whose occupancy measure $\rho_\pi$ minimizes the divergence between $\rho_{\pi_{c}}$ and $\rho_{\pi_{e}}$, while maximizing the divergence from $\rho_{\pi_{m}}$. 
The objective function is defined as follows:
\begin{equation}
\begin{aligned}
\min_{\pi} & \left(-\lambda H(\pi) + \alpha \psi^*(   \rho_{\pi} - \rho_{\pi_{c}}) + \beta \psi^*( \rho_{\pi} - \rho_{\pi_{e}})\right. \\
& \left.-\gamma \psi^*( \rho_{\pi} - \rho_{\pi_{m}})\right),
\end{aligned}
\label{first obj func}
\end{equation}
where $\psi^*$ represents the Jensen-Shannon divergence and $H(\pi) \triangleq \mathbb{E}_\pi[-\log \pi(a|s)]$ is the causal entropy of the policy $\pi$. This entropy term encourages exploration and helps prevent the policy from collapsing into a deterministic strategy. The parameters $\lambda$, $\alpha$, $\beta$, and $\gamma$ control the influence of causal entropy and the distances between the occupancy measures of the agent and different experts. 
In experiments, we assign a larger weight to $\alpha$ to prioritize CE trajectories, which allows the model to leverage the earliness of EAE while remaining anchored by the stability of CE, offering timely and accurate predictions instead of being excessively cautious or aggressive. And we ensure that EAE and ME trajectories contribute equally, preventing the policy from deviating too far in either direction for a more balanced learning process.

This objective function ensures that the learned policy incorporates elements from all three expert policies, as intuitively illustrated in Figure~\ref{fig:MTE example}. The \textit{EAE} begins with a high probability of taking a stop action, while the \textit{CE} starts with a lower initial probability that gradually increases due to its cautious nature. The agent aims to align itself with both \textit{EAE} and \textit{CE}. 
However, at certain points (e.g., time $t=6$ in the Figure \ref{fig:MTE example}), the agent may become uncertain about which decision to make, leading to random choices. Such undesired behaviors can result in incorrect actions, such as stopping too early or continuing unnecessarily. By factoring in the \textit{ME} policy, the agent is encouraged to deviate from the state-action distribution of \textit{ME}. Consequently, during uncertain time points, the agent’s probability of stopping will not hover around $0.5$, and will favor either the continue or stop action more decisively.  We also provided the theoretical justification for this objective function based on the generative adversarial imitation learning framework in the Appendix~\ref{Proof of objective func}.

Following~\citet{ho2016generative}, $\psi^*$ is defined as:
\begin{equation}
\label{GAIL distance measure}
\begin{aligned}
& \psi^*(\rho_\pi - \rho_{\pi_E}) = \max \limits_{D} \left\{\mathbb{E}_\pi[\log D(s,a)] + \mathbb{E}_{\pi_E}[\log(1-D(s,a))]\right\},
\end{aligned}
\end{equation}
where $D \in (0,1)^{\mathcal{S} \times \mathcal{A}}$ is the discriminative classifiers responsible for distinguishing between state-action pairs of agent policy and those of the expert. This equation represents the optimal negative log loss of this binary classification problem of differentiating agent and expert trajectories. This optimal loss is the Jensen-Shannon divergence $JSD(\rho_\pi, \rho_{\pi_E})$~\cite{ho2016generative}, which provides a distance metric between occupancy measures. The agent aims to minimize the divergence between its state-action distribution and those of \textit{CE} and \textit{EAE}, while maximizing the divergence from \textit{ME}.

\subsubsection{Training} \label{Training}
To solve the objective function in Eq.~\eqref{first obj func}, following~\citet{ho2016generative}, our agent consists of three neural networks: a policy network $\pi_\theta$ parameterized by $\theta$, a discriminator network $D_\omega$ parameterized by $\omega$, and a value network $V_{\phi}$ parameterized by $\phi$. During training, the agent alternates between collecting trajectories using the current policy and updating the policy via Proximal Policy Optimization (PPO) method \cite{schulman2017proximal}. The discriminator distinguishes expert trajectories from agent-generated ones and provides an implicit reward signal. Since PPO method requires estimating the advantage of each action~\cite{schulman2017proximal}, the value network serves as the critic and estimates the state-value function $V_{\phi}(s)$, stabilizing policy updates. 

 With Eq.~\eqref{first obj func} and Eq.~\eqref{GAIL distance measure}, the objective function can be written as:

\begin{equation}
\label{ME-GAIL-obj}
\begin{aligned}
L(\omega, \theta) & \doteq \min_{\theta} \max_{\omega} \left\{ - \lambda H(\pi_\theta) + (\alpha + \beta - \gamma ) \mathbb{E}_{\pi_\theta} \left[ \log D_\omega(s, a) \right] \right. + \\ 
& \quad \alpha \mathbb{E}_{\pi_{c}} \left[ \log(1 - D_\omega(s, a)) \right] + \beta \mathbb{E}_{\pi_{e}} \left[ \log(1 - D_\omega(s, a)) \right] \\
& \quad \left. -\gamma \mathbb{E}_{\pi_{m}} \left[ \log(1 - D_\omega(s, a)) \right]  \right\}
\end{aligned}
\end{equation}

Our approach aims to compute a saddle point $(\theta, \omega)$ for the above objective function, by updating the parameters of policy and discriminator sequentially \cite{achiam2017constrained,shao2024imitating,satija2020constrained}:

\textbf{Updating $\omega$:}
The gradient of Eq.~\eqref{ME-GAIL-obj} with respect to $\omega$ is calculated as:
\begin{equation}
\label{gradient of discriminator}
\begin{aligned}
 & \bigtriangledown_\omega L(\omega,\theta) = (\alpha + \beta - \gamma )\mathbb{E}_{\pi_\theta}[\bigtriangledown_{\omega}\log D_\omega(s,a)] \\
& + \alpha \mathbb{E}_{\pi_{c}}[\bigtriangledown_{\omega} \log(1-D_\omega(s,a))]   \\
& + \beta \mathbb{E}_{\pi_{e}}[\bigtriangledown_{\omega} \log(1-D_\omega(s,a))] -  \gamma 
 \mathbb{E}_{\pi_{m}}[\bigtriangledown_{\omega} \log(1-D_\omega(s,a))]
 \end{aligned}
\end{equation}
We use Adam optimizer~\cite{kingma2014adam} to update the variable $\omega$, targeting the maximization of Eq. \eqref{ME-GAIL-obj} with respect to the discriminator $D_\omega$.

\textbf{Updating $\theta$:}
We use PPO to update policy network parameters $\theta$. PPO-clip updates policy via
$ \theta_{k+1} = \arg \max_{\theta} \, \mathbb{E}_{s, a \sim \pi_{\theta_k}} \left[ \mathcal{L}(s, a, \theta_k, \theta) \right]$.
The objective function $\mathcal{L}$ is defined as: 
\begin{equation}
\label{PPO-clip}
\begin{aligned}
 &\mathcal{L}(s, a, \theta_k, \theta) = \\ 
 & \min \left( \frac{\pi_{\theta}(a|s)}{\pi_{\theta_k}(a|s)} A^{\pi_{\theta_k}}(s, a), \, \text{clip} \left( \frac{\pi_{\theta}(a|s)}{\pi_{\theta_k}(a|s)}, 1 - \epsilon, 1 + \epsilon \right) A^{\pi_{\theta_k}}(s, a) \right),
\end{aligned}
\end{equation}
where $\epsilon$ is a small hyperparameter that limits how far the new policy $\pi_{\theta}$ can deviate from the old policy $\pi_{\theta_k}$, and $A^{\pi_{\theta_k}}$ is the Advantage that measures the difference between the Q-value and the value function $V_{\phi}$ parameterized by $\phi$. Since no explicit reward function is available to compute the Advantage $A^{\pi_{\theta_k}}$, we use the output of the discriminator, $-\log D_\omega(s,a)$, as the implicit reward signal. We then apply Generalized Advantage Estimation (GAE)~\cite{schulman2015high} to compute $A^{\pi_{\theta_k}}(s, a)$.

\textbf{Updating $\phi$:}
The value network parameters $\phi$ is updated by minimizing the mean-squared error of the value function, using the following objective: 
\begin{equation}
\label{gradient of reward, cost value }
\begin{aligned}
&  \min \limits_{\phi} \mathop{\mathbb{E}} \limits_{s_t \sim \pi_{\theta_k}} (V_{\phi}(s_t)-\hat{R}_t)^2, \\
\end{aligned}
\end{equation}
where $\hat{R}_t$ represents the reward-to-go, which is also computed using the GAE method~\cite{schulman2015high}.

Initially, the policy parameters $\theta$, value network parameters $\phi$, and discriminator network parameters $\omega$ are randomly initialized. The agent alternates among collecting trajectories, updating the policy using PPO, and refining both the value network and discriminator through gradient-based updates. This iterative process continues until convergence, resulting in an optimal policy $\pi_\theta$ that effectively synthesizes the policies of different experts. By integrating different type of expert trajectories, our approach provides a flexible and effective framework for early rumor detection, particularly in scenarios where the optimal strategy requires balancing diverse experts behaviors. Training algorithm is in Algorithm~\ref{Algo: ME-GAIL}.

\subsection{Inference}
After training, the trained policy network $\pi_\theta$ is directly deployed for early time point prediction. Note that expert trajectory generation are required only for training. During inference, the agent continuously monitors social media posts, deciding whether to stop observation based on the posts available up to that time point through the policy network. When it determines to stop, the LLM is activated to detect rumors using the observed posts.

\begin{table*}[t!]
    \small
    \centering
     \begin{adjustbox}{width={1.\linewidth},keepaspectratio}%
        \begin{tabular}{lccccccccccc}
        
        \toprule[1.0pt]
          & \multicolumn{2}{c}{\textbf{PHEME}} && \multicolumn{2}{c}{\textbf{TWITTER}} && \multicolumn{2}{c}{\textbf{BEARD}} && \multicolumn{2}{c}{\textbf{Twitter-COVID-19}} \\ 
          \cline{2-3}\cline{5-6}\cline{8-9}\cline{11-12}
          & \textbf{macro-F1} & \textbf{ER} && \textbf{macro-F1} & \textbf{ER} && \textbf{macro-F1} & \textbf{ER} && \textbf{macro-F1} & \textbf{ER} \\ 
        \midrule[0.5pt]
        \multirow{2}{*}{$\textbf{Llama3}_{\textit{(first post)}}$}
                               & \multirow{2}{*}{$0.457_{(0.016)}$} & \multirow{2}{*}{-} && \multirow{2}{*}{$0.593_{(0.029)}$} & \multirow{2}{*}{-} && \multirow{2}{*}{$0.505_{(0.016)}$} & \multirow{2}{*}{-} && \multirow{2}{*}{$0.378_{(0.023)}$} & \multirow{2}{*}{-} \\
        \specialrule{0em}{4pt}{1pt}
        \multirow{2}{*}{$\textbf{Mistral}_{\textit{(first post)}}$}
                               & \multirow{2}{*}{$0.571_{(0.024)}$} & \multirow{2}{*}{-} && \multirow{2}{*}{$0.543_{(0.028)}$} & \multirow{2}{*}{-} && \multirow{2}{*}{$0.443_{(0.033)}$} & \multirow{2}{*}{-} && \multirow{2}{*}{$0.493_{(0.027)}$} & \multirow{2}{*}{-} \\
        \specialrule{0em}{4pt}{1pt}     
         \multirow{2}{*}{$\textbf{ChatGPT}_{\textit{(first post)}}$}
                               & \multirow{2}{*}{${0.604}_{(0.013)}$} & \multirow{2}{*}{-} && \multirow{2}{*}{$0.535_{(0.031)}$} & \multirow{2}{*}{-} && \multirow{2}{*}{$0.445_{(0.028)}$} & \multirow{2}{*}{-} && \multirow{2}{*}{${0.578}_{(0.015)}$} & \multirow{2}{*}{-} \\
                          
        \specialrule{0em}{4pt}{1pt}   
        \midrule[0.5pt]
        \multirow{2}{*}{$\textbf{ERD}$}
         & \multirow{2}{*}{$0.487_{(0.105)}$} & \multirow{2}{*}{$1.000_{(0.000)}$} && \multirow{2}{*}{$0.571_{(0.017)}$} & \multirow{2}{*}{$1.000_{(0.000)}$} && \multirow{2}{*}{${\mathbf{0.572}_{(0.101)}}$} & \multirow{2}{*}{$1.000_{(0.000)}$} && \multirow{2}{*}{$0.528_{(0.059)}$} & \multirow{2}{*}{$1.000_{(0.000)}$} \\
         \specialrule{0em}{4pt}{1pt}
         \multirow{2}{*}{$\textbf{CED}$}
         & \multirow{2}{*}{$0.453_{(0.041)}$} & \multirow{2}{*}{$0.978_{(0.011)}$} && \multirow{2}{*}{$0.654_{(0.044)}$} & \multirow{2}{*}{$0.617_{(0.040)}$} && \multirow{2}{*}{${0.543_{(0.088)}}$} & \multirow{2}{*}{$0.553_{(0.041)}$} && \multirow{2}{*}{$0.524_{(0.070)}$} & \multirow{2}{*}{$0.421_{(0.067)}$} \\
         \specialrule{0em}{4pt}{1pt}
        \multirow{2}{*}{$\textbf{HEARD}$}
         & \multirow{2}{*}{$0.501_{(0.038)}$} & \multirow{2}{*}{$0.266_{(0.006)}$} && \multirow{2}{*}{$0.627_{(0.037)}$} & \multirow{2}{*}{$0.647_{(0.305)}$} && \multirow{2}{*}{${0.553_{(0.044)}}$} & \multirow{2}{*}{$0.337_{(0.017)}$} && \multirow{2}{*}{$0.530_{(0.068)}$} & \multirow{2}{*}{$\mathbf{0.175}_{(0.013)}$} \\
         \specialrule{0em}{4pt}{1pt}
         \midrule[0.5pt]
        
        \multirow{2}{*}{\textbf{Ours (w/ Llama3)}}
                               & \multirow{2}{*}{\textbf{$ {  0.478_{(0.029)}}$}} & \multirow{2}{*}{$0.283_{(0.156)}$} && \multirow{2}{*}{\textbf{$ { \mathbf{0.705}_{(0.039)}}$}} & \multirow{2}{*}{$0.530_{(0.405)}$} && \multirow{2}{*}{\textbf{$ {0.546_{(0.031)}}$}} & \multirow{2}{*}{$0.493_{(0.367)}$} && \multirow{2}{*}{\textbf{${0.544_{(0.032)}}$}} & \multirow{2}{*}{$0.846_{(0.007)}$} \\

        \specialrule{0em}{4pt}{1pt}
       
        \multirow{2}{*}{\textbf{Ours (w/ Mistral)} }
                               & \multirow{2}{*}{\textbf{${0.590_{(0.012)}}$}} & \multirow{2}{*}{$0.276_{(0.125)}$} && \multirow{2}{*}{\textbf{${0.630_{(0.039)}}$}} & \multirow{2}{*}{$\mathbf{0.043}_{(0.038)}$} && \multirow{2}{*}{\textbf{${0.488_{(0.030)}}$}} & \multirow{2}{*}{$\mathbf{0.130}_{(0.073)}$} && \multirow{2}{*}{\textbf{${0.511_{(0.056)}}$}} & \multirow{2}{*}{${0.184}_{(0.365)}$} \\   
        \specialrule{0em}{7pt}{1pt}

        \multirow{1}{*}{\textbf{Ours (w/ ChatGPT)}}
                               & \multirow{1}{*}{\textbf{${\mathbf{0.646}_{(0.009)}}$}} & \multirow{1}{*}{$\mathbf{0.205}_{(0.043)}$} && \multirow{1}{*}{\textbf{${{0.702}_{(0.068)}}$}} & \multirow{1}{*}{${0.526}_{(0.402)}$} && \multirow{1}{*}{\textbf{${0.546_{(0.045)}}$}} & \multirow{1}{*}{$0.545_{(0.279)}$} && \multirow{1}{*}{\textbf{${\mathbf{0.586}_{(0.031)}}$}} & \multirow{1}{*}{$0.366_{(0.285)}$} \\ 

        \bottomrule[1.0pt]
        \end{tabular}
    \end{adjustbox}
  \caption{Results of existing EARD methods, LLMs using the first post (\textit{with Earliest Rumor Detection Strategy}), and our framework integrated with base LLMs across four datasets. Standard deviation is in (.). The best results are in \textbf{bold}.}
  \label{tab:main results}
  \vspace{-3ex}
\end{table*}

\begin{table*}[ht!]
    \small
    \centering
    \begin{adjustbox}{width={1.\linewidth},keepaspectratio}%
        \begin{tabular}{l|c|cccccccc}
        
        \toprule[1.0pt]
        {\textbf{Training Dataset}} 
        & &\textbf{ERD} & &\textbf{CED} & &\textbf{HEARD} & \textbf{Ours (w/ Llama3)} & \textbf{Ours (w/ Mistral)} & \textbf{Ours (w/ ChatGPT)} \\ 
        
        \midrule[0.5pt]   

        \multirow{2}{*}{\color{gray}{\textbf{Twitter-COVID-19}}} & \color{gray}{\textbf{macro-F1}} &  \color{gray}{$0.528_{(0.059)}$} & & \color{gray}{$0.524_{(0.070)}$} && \color{gray}{$0.530_{(0.068)}$} & \color{gray}{$0.544_{(0.032)}$}  & \color{gray}{$0.511_{(0.056)}$}  & \color{gray}{$0.586_{(0.031)}$} \\
        \specialrule{0em}{1pt}{1pt}
        \cline{2-10}
        \specialrule{0em}{1pt}{1pt}
        & \color{gray}{\textbf{ER}} & \color{gray}{$1.000_{(0.000)}$} && \color{gray}{$0.421_{(0.067)}$} && \color{gray}{$0.175_{(0.013)}$} & \color{gray}{$0.846_{(0.007)}$} & \color{gray}{$0.184_{(0.365)}$}  & \color{gray}{$0.366_{(0.285)}$} \\
        \midrule[0.5pt]   
        
        \multirow{3}{*}{\textbf{PHEME}} & \textbf{macro-F1} &  $0.316_{(0.099)}$ && $0.394_{(0.087)}$ && $0.341_{(0.059)}$ & $0.442_{(0.046)}$ & $0.497_{(0.034)}$ & $\mathbf{0.521}_{(0.020)}$ \\
        \specialrule{0em}{1pt}{1pt}
        \cline{2-10}
        \specialrule{0em}{1pt}{1pt}
        & \textbf{ER} & $1.000_{(0.000)}$ && $0.311_{(0.041)}$ && $0.025_{(0.003)}$ & $0.086_{(0.114)}$ & $0.030_{(0.022)}$ & $\mathbf{0.018}_{(0.006)}$ \\
        \midrule[0.5pt]   
        
        \multirow{3}{*}{\textbf{TWITTER}} & \textbf{macro-F1} &  $0.473_{(0.037)}$ && $0.482_{(0.055)}$ && $0.480_{(0.040)}$ & $0.461_{(0.092)}$ & $0.499_{(0.021)}$ & $\mathbf{0.546}_{(0.018)}$ \\
        \specialrule{0em}{1pt}{1pt}
        \cline{2-10}
        \specialrule{0em}{1pt}{1pt}
        & \textbf{ER} & $1.000_{(0.000)}$ & &$0.378_{(0.054)}$ && $0.641_{(0.272)}$ & $0.186_{(0.368)}$ & $\mathbf{0.020}_{(0.017)}$ & $0.434_{(0.390)}$ \\
        
        \bottomrule[1.0pt]
        \end{tabular}
    \end{adjustbox}
  \caption{Cross-dataset evaluation results on Twitter-COVID-19 test set, with models trained on the PHEME and TWITTER datasets (both collected prior to the COVID-19 pandemic). Standard deviation is in (.). The best results are in \textbf{bold}.}
  \label{tab:exp2}
  \vspace{-4ex}
\end{table*}

\section{Experiments and Results}
\begin{algorithm}[!t]
\caption{}
\label{Algo: ME-GAIL}
\begin{flushleft}
\noindent \textbf{Input}: Initial parameters of policy $\theta$, value network $\phi$, discriminator network $\omega$, time steps per iteration $K$, a set of expert trajectories $\Phi_E = \{\tau_E \sim \pi_E\}, E = \{c, e, m\}$,  entropy parameter $\lambda$, learning rates $\alpha_r,\alpha_\omega$. \\
\textbf{Output}: Optimal policy $\pi_\theta$ 
\end{flushleft}
\begin{algorithmic}[1] 
\FOR {$k = 1,2,...$}
\STATE Collect a set of learner's trajectories $\Phi_k = \{\tau_i \}$ by running policy $\pi_{\theta_k}$ for $K$ time steps.
\STATE Collect the implicit reward signal $r_t$ of $K$ time steps by using the discriminator output: $r_t = -\log (D_\omega(s_t,a_t))$ 
\STATE Compute $V_{\phi}(s_t)$ of $K$ time steps.
\STATE Compute the advantage $A^{\pi_{\theta_k}}$, reward-to-go $\hat{R}_t$  of $K$ time steps by using GAE.
\STATE Update policy by using PPO-Clip: 
\[
\theta' = \arg \max_{\theta} \frac{1}{K} \sum_{t=1}^K \mathcal{L}(s_t, a_t, \theta_k, \theta)
\]
\STATE Update value network:
\[
\phi' \gets \phi - \frac{1}{K}  \sum_{t=1}^K \alpha_r \nabla_{\phi} (V_{\phi}(s_t)-\hat{R}_t)^2 
\]

\STATE Update discriminator network: 
\[
\omega' \gets \omega -  \frac{1}{K} \sum_{t=1}^K \alpha_\omega \left( \bigtriangledown_\omega -L(\omega,\theta) \right)
\]
\STATE $ \theta \gets \theta'$, $\phi \gets \phi'$, $\omega \gets \omega'$.
\ENDFOR 
\end{algorithmic}
\end{algorithm}

\subsection{Experiments Setup}
\subsubsection{Datasets.}
Four real-world datasets are used: 1) \textbf{PHEME}~\cite{zubiaga2016learning}: This dataset collects posts from conversation threads during newsworthy events; 2) \textbf{TWITTER}~\cite{ma2016detecting}: This dataset provides posts of events sourced from a fact-checking website, widely used for rumor detection studies; 3) \textbf{Twitter-COVID-19}~\cite{lin-etal-2022-detect}: This dataset gathers posts from conversation threads discussing events about the COVID-19 pandemic~\cite{DBLP:journals/corr/abs-2010-06906}; 4) \textbf{BEARD}~\cite{zeng-gao-2022-early}: this dataset is for EARD task, focusing on covering early-stage posts relevant to various events. 
Each dataset offers a chronological collection of relevant posts for each event, annotated as either rumor or non-rumor.

\subsubsection{Baselines.} \label{baselines}
We compare with existing full-shot EARD methods under few-shot setup, including ERD~\cite{zhou2019early} that uses a Deep Q-Network (DQN) to enforce the model to focus on early posts; CED~\cite{song2019ced} that uses a fixed probability threshold to check if the prediction is credible for determining early detection point; HEARD~\cite{zeng-gao-2022-early} that automatically determines a detection point at which predictions in future time steps remain unchanged using the neural Hawkes Process~\cite{mei-etal-2017-The}. 
Given the strong capability of LLMs across NLP tasks, we also include Mistral~\cite{Mistral}, Llama3~\cite{llama3}\footnote{We refer to Mistral-Instruct-v0.3 as Mistral and Llama-3-Instruct as Llama3 for brevity.}, ChatGPT~\cite{ChatGPT} as baselines. We incorporate these LLMs using two early detection strategies. \emph{Earliest Rumor Detection Strategy}: The LLMs are prompted to make predictions using only the \textit{first} post of each instance, following~\citet{miao2021syntax}. \emph{Preset Time Checkpoints Strategy}: we use a set of preset time checkpoints at $\{1\text{h}, 6\text{h}, 12\text{h}, 24\text{h}, 36\text{h}\}$. All posts published before each checkpoint are provided to the LLMs for prediction. The timestamp of the first post serves as the starting point, and the time elapsed for each subsequent post is calculated relative to this initial timestamp. This strategy is commonly used in general rumor detection studies for evaluating early detection performance~\citep{ma2016detecting,yu2017convolutional,ma2017detect,ma-etal-2018-rumor,guo2018rumor,bian2020rumor,Lin_Yi_Ma_Jiang_Luo_Shi_Liu_2023}. 

\subsubsection{Experimental Settings.}
We randomly sample 50 labeled training instances for trajectory generation and 100 \textit{unlabeled} instances to serve as the environment for the agent, allowing the agent to dynamically explore and adapt its strategy in an online policy learning setting. We utilize the pretrained BERTweet model~\cite{nguyen-etal-2020-bertweet} to obtain embeddings of each state. Following~\citet{ho2016generative}, we employ two hidden layers with a \textit{Tanh} activation function for the policy, value, and discriminator networks. The coefficient of causal entropy $\lambda$ is set to 0.01. We set the expert ratios to $\alpha=0.7$, $\beta=0.15$, and $\gamma=0.15$ as discussed in \S \ref{Objective Function}. 

We utilize the pretrained Mistral with 7B parameters\footnote{\url{https://huggingface.co/mistralai/Mistral-7B-Instruct-v0.3}} and Llama3 with 8B parameters\footnote{We choose open-source LLMs under the limits of our computing resources. Scaling up to larger models would impose huge computational costs, as each experiment requires five runs. For example, a baseline experiment using the Preset Time Checkpoints Strategy needs 25 runs for one dataset.}. For ChatGPT, we use the API service from openAI\footnote{\url{https://openai.com/index/gpt-4o-mini-advancing-cost-efficient-intelligence/}}.
As mentioned in \S \ref{problem definition}, we do not consider a development set to reflect a realistic few-shot setting, ensuring consistent settings across all models and datasets. We use the original source codes released by the baseline EARD methods. 

We report macro-F1 score for classification and use Early Rate (ER)~\cite{song2019ced} to measure the proportion of posts used for detection. ER measures the proportion of posts used for detection:
$ER=\frac{1}{|\mathbf{C}_\text{test}|} 
\sum_{C \in \mathbf{C}}\frac{i_{C}}{|C|}$,
where $\mathbf{C}_\text{test}$ represents the test set, $i_{C}$ indicates the early detection decision made at the $i$-th post in instance $C$, and $|C|$ is the total number of posts in that instance. \emph{Lower} ER means that the model can detect rumors \emph{earlier}. For a fair comparison, we use the same prompt template for all models and run each experiment five times with different random seeds, reporting the mean and standard deviation for each metric. The prompt used for LLM prediction is ``\textit{Analyze the given sequence of social media posts, determine if it is a rumor. Respond Yes or No only.
Posts: [the list of observed posts in chronological order]}''.

Table \ref{experiment hyperparameters} presents hyper-parameters in our experiments. Following~\citet{ho2016generative}, the policy, value, and discriminator networks are initialized randomly. We set the number of training steps to 200,000, with a batch size of 64 for the discriminator network and 4 for the policy and value networks.

\begin{table}[hbt]
    \centering
    \caption{Hyper-parameters in experiments}
    \begin{tabular}{lc}
        \toprule
        \small hyper-parameter & value \\
        \midrule
       Policy and Value network hidden size & (64, 64) \\
       Activation & Tanh \\
       Time steps per iteration $K$  & 200 \\
       Total time steps & 200,000 \\
       Generator epoch & 4 \\
       Discriminator epoch & 5\\
       Generalized Advantange Estimation $\bar{\gamma}$ & 0.99\\
       Generalized Advantange Estimation $\lambda_0$ & 0.97\\
       Learning rate( Value network) & 3 $\times$  $10 ^ {-4}$ \\
       Policy entropy coefficient $\lambda $& 0.01 \\
       Clip parameter $\epsilon$ & 0.1 \\
       \bottomrule
    \end{tabular}
   \label{experiment hyperparameters}
\end{table}

\begin{figure*}
    \includegraphics[width=.9\textwidth]{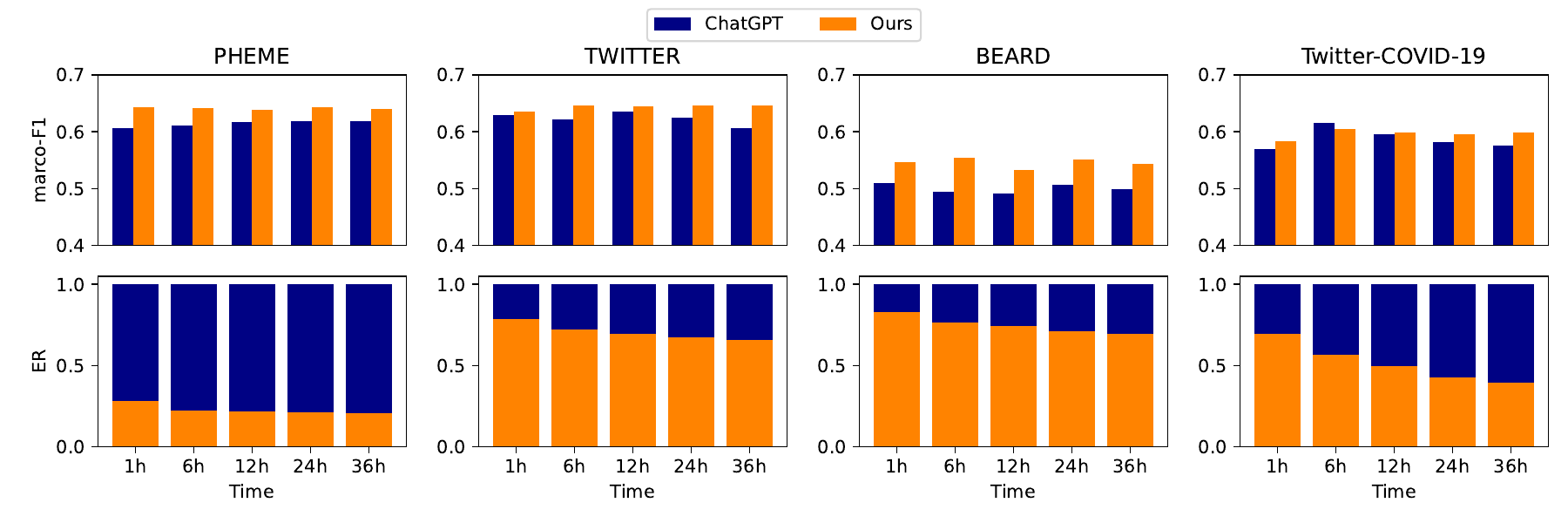}
    \vspace{-5ex}
\caption{Comparison of ChatGPT (\textit{with Preset Time Checkpoints Strategy}) and our method. The x-axis indicates different time intervals of posts used for prediction. The top row displays macro-F1 scores, while the bottom row shows the ER.}
    \label{fig:cttoff across gpt}
    \vspace{-2ex}
\end{figure*}

\subsection{Results and Analysis}
\subsubsection{Main Results.}\label{sec:main results}
In Tables~\ref{tab:main results} and \ref{tab:exp2}, we firstly present the results of existing EARD methods, LLMs using the earliest rumor detection strategy, and LLMs integrated with our proposed framework. We have the following observations.

\textbf{First,} as shown in Table~\ref{tab:main results}, despite using only first post following the earliest rumor detection strategy~\cite{miao2021syntax}, LLMs achieve competitive or superior macro-F1 scores across multiple datasets, compared to existing EARD methods. Specifically, ChatGPT outperforms all existing EARD methods by up to 33.3\% and 10.3\% on PHEME and Twitter-COVID-19, respectively. Even smaller LLMs like Mistral surpass all existing EARD methods on PHEME, and Llama3 outperforms ERD on TWITTER dataset. These results highlight the potential of LLMs used as detectors in EARD, particularly in data-scarce scenarios.

\textbf{Second,} our framework effectively enhances the performance of base LLMs, by utilizing our lightweight agent that automatically determines an early detection point. By continuously monitoring the social media post stream and capturing the evolving dynamics of rumors, the agent enables base LLMs to leverage richer information for more accurate predictions. In Table~\ref{tab:main results}, we observe that the relative improvements are 8.2\% for Mistral, 15.6\% for ChatGPT, and 18.9\% for Llama3, averaged across four datasets. These consistent improvements across different LLMs, suggest our framework is effective, robust and model-agnostic.

\textbf{Third,} our method achieves overall better performance across both ER and F1, compared to existing EARD methods. For instance, in Table~\ref{tab:main results}, on the TWITTER dataset, our framework enables Mistral to surpass all EARD baselines in F1 score while achieving a significantly lower ER (0.043) compared to ERD (1.000), CED (0.617), and HEARD (0.647). We also observe that ERD fails to make early decisions in few-shot setting, as shown by its high ER scores. This is likely due to its reliance on a manually designed reward function, where its policy trained by DQN is highly sensitive. Specifically, the reward function applies only a small penalty for continuation but imposes a large penalty for incorrect early termination~\cite{zhou-etal-2019-early}, discouraging the model to stop early. This behavior aligns with previous findings~\cite{zeng-gao-2022-early}, and the problem is further exacerbated in the few-shot scenario, where data scarcity undermines the stability and effectiveness of reinforcement learning.

\textbf{Lastly,} we examine how well existing EARD methods and our method generalize to \textit{unseen new event data} by training the models using PHEME and TWITTER datasets (collected prior to COVID-19) and test them on the Twitter-COVID-19. In Table~\ref{tab:exp2}, we observe a general performance drop when training on the PHEME and TWITTER compared to training directly on the Twitter-COVID-19 dataset. 
However, under this more challenging setting, our method demonstrates stronger performance in both detection accuracy and earliness compared to all EARD baselines. And the overall performance degradation of baselines is more substantial than that of our approach. 
For instance, while Ours (w/ Mistral) slightly underperforms ERD, CED, and HEARD when trained directly on Twitter-COVID-19, it surpasses all of them in F1 score and achieves a lower ER when trained on PHEME and TWITTER. Specifically, ERD, CED, and HEARD suffer relative F1 drops of 67.8\%, 33.1\%, and 55.4\% when trained on PHEME, and relative 11.6\%, 8.7\%, and 10.4\% drops when trained on TWITTER. In contrast, Ours (w/ Mistral) incurs only 2.8\% and 2.4\% degradation under the same settings. These results highlight the superior generalization capability of our approach in data-limited scenarios, emphasizing its robustness in detecting rumors from previously unseen events.

\subsubsection{Comparison with the LLMs using Preset Time Checkpoints Strategy.} 
Instead of only using the first post, this strategy provides all posts published before each time checkpoint to the LLMs for detection. We conduct this experiment based on ChatGPT, given its best overall performance across the four datasets. As shown in Figure~\ref{fig:cttoff across gpt}, our method consistently improves the detection performance of ChatGPT, while achieving a lower ER. For example, on the Twitter-COVID-19 dataset, ChatGPT’s ER is reduced by more than 50\%, dropping to below $0.5$ at the 12-hour checkpoint with our agent. This shows that, with the automatic early time point determination provided by our agent, the LLMs can achieve better early detection performance with far fewer posts. Our imitation learning algorithm enables the agent to effectively mimic the \textit{CE} and \textit{EAE}, while also learning to avoid the misleading behaviors of \textit{ME}, suggesting that learning from multiple experts allows for both earlier and more accurate predictions. Moreover, the improvement on the EARD-oriented BEARD dataset looks greater than those on the other rumor datasets, indicating that our method effectively leverages early-stage information, which makes it particularly well-suited for the EARD task. We also provide the results of other LLMs in \S~\ref{appendix:Preset Time Checkpoints}, 
observing similar trends. 

\begin{figure}[ht!]
\includegraphics[width=1.0\linewidth]{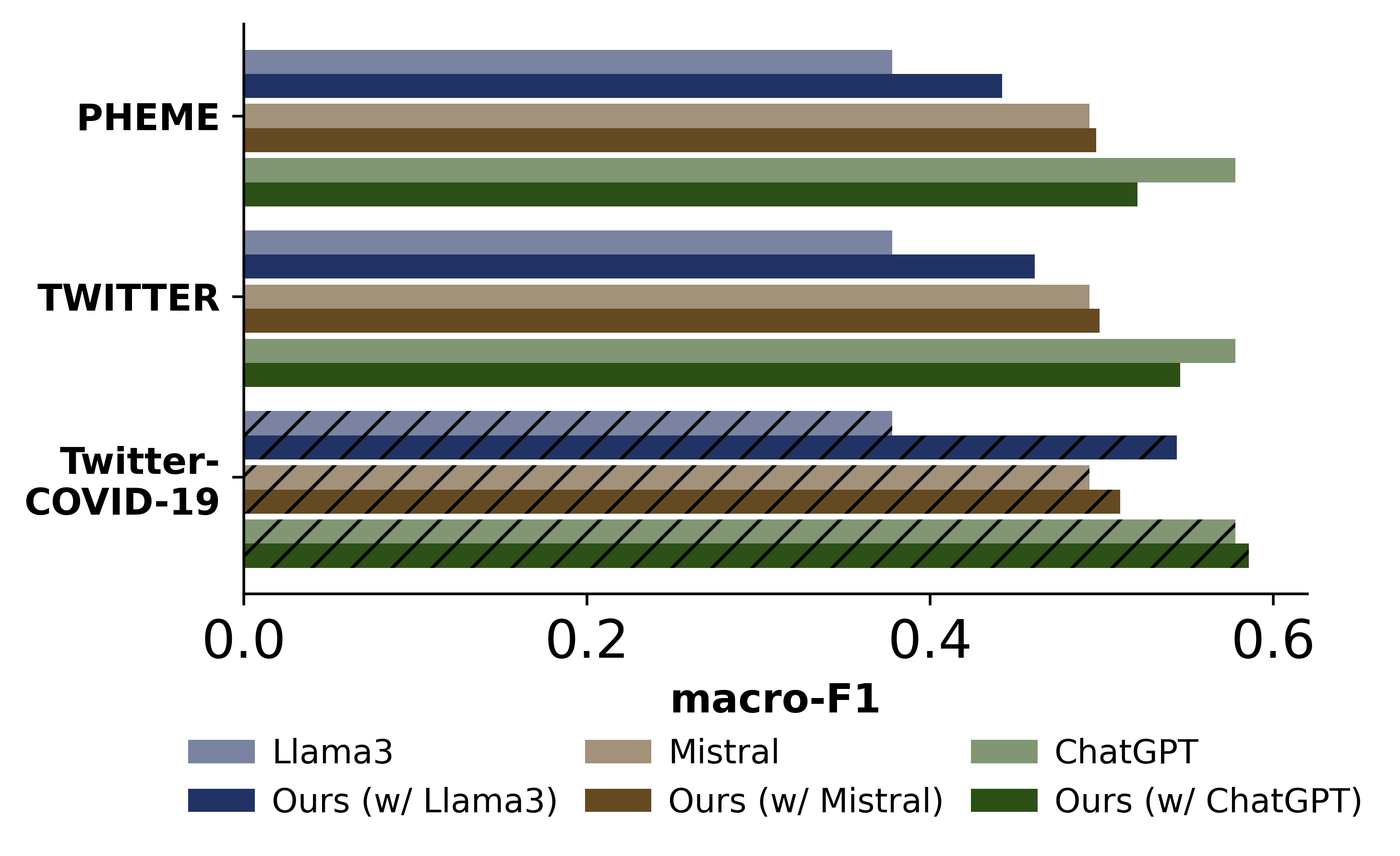}
\caption{Results of base LLMs evaluated directly and the LLMs incorporated with our agent trained on PHEME and TWITTER, which are tested on Twitter-COVID-19. Bars with diagonal hatching (``/'') indicate models trained directly on Twitter-COVID-19.}
    \label{fig:ood improvement gpt} 
\end{figure}

\subsubsection{Analysis on How Generalizability Varies Across LLMs} \label{exp1}
As shown in Table~\ref{tab:exp2} and discussed in \S~\ref{sec:main results}, our method generalizes better than the baseline EARD approaches to new event data. Trained on pre-COVID datasets (PHEME and TWITTER) and tested on Twitter-COVID-19, it consistently outperforms baselines in both F1 and ER. We further investigate how this generalizability varies across different base LLMs. As shown in Figure~\ref{fig:ood improvement gpt}, our method dramatically boosts Llama3 by around 17\% and 18\% when trained on PHEME and TWITTER, respectively. This suggests that the agent can learn common expert decision patterns encoded in expert trajectories from other datasets, which benefits LLMs when applied to new event data. These gains are less pronounced with Mistral, which outperforms Llama3, and turn into a negative influence with ChatGPT, which initially shows the best performance on this COVID-related dataset. We conjecture that as the base model’s initial performance improves, the benefits of learned common behavior patterns become less useful compared to the information specific to the new event data. Thus, applying an agent trained on other datasets could lead to suboptimal decisions. One possible solution is to train the agent on a small number of new event data. As shown in Figure~\ref{fig:ood improvement gpt}, even a few samples from the Twitter-COVID-19 dataset can improve performance across all base LLMs. Another promising direction is to model this problem as a Partially Observable Markov Decision Process (POMDP)~\cite{sondik1971optimal,spaan2012partially}, which is useful for capturing 
latent uncertainties and incomplete observations in sequential decision-making \cite{kaelbling1998planning} and may better capture the uncertainties in scenarios involving new event data. We leave this for future work.

\begin{table}[ht]
\centering
\begin{tabular}{lcccc}
\toprule
\textbf{Dataset} & 10 shots & 20 shots & 30 shots & 40 shots \\
\midrule
\textbf{PHEME} & 8.86\% $\uparrow$ & 8.17\% $\uparrow$  & 8.16\% $\uparrow$  & 8.92\% $\uparrow$  \\
\textbf{TWITTER} & 26.19\% $\uparrow$  & 27.64\% $\uparrow$  & 25.17\% $\uparrow$ & 23.46\% $\uparrow$  \\
\textbf{BEARD} & 18.26\% $\uparrow$  & 22.37\% $\uparrow$  & 18.88\% $\uparrow$  & 20.61\% $\uparrow$  \\
\textbf{Twitter-COVID-19} & 1.37\% $\uparrow$ & 2.82\% $\uparrow$ & 0.92\% $\uparrow$ & 3.99\% $\uparrow$ \\
\bottomrule
\end{tabular}
\caption{Relative F1 improvement of our method over ChatGPT used as the base LLM under varying number of shots.}
\label{tab:shots main text}
\vspace{-6ex}
\end{table}

\subsubsection{Analysis on Robustness to Limited Supervision} \label{main text: num of shots}
To evaluate robustness under limited labeled data, we reduce the number of training shots and report results of our method using ChatGPT as the base model in Table~\ref{tab:shots main text}, showing that our method remains effective even with scarce labeled data. Across all datasets, it consistently outperforms base LLMs, achieving relative F1 improvements up to 27.64\%. Even with only 10 shots, our approach maintains stable gains, demonstrating strong capability with only a minimal supervision. The detailed results are provided in Appendix~\ref{appendix:shots}.

\begin{figure}[h]
    \includegraphics[width=0.9\linewidth]{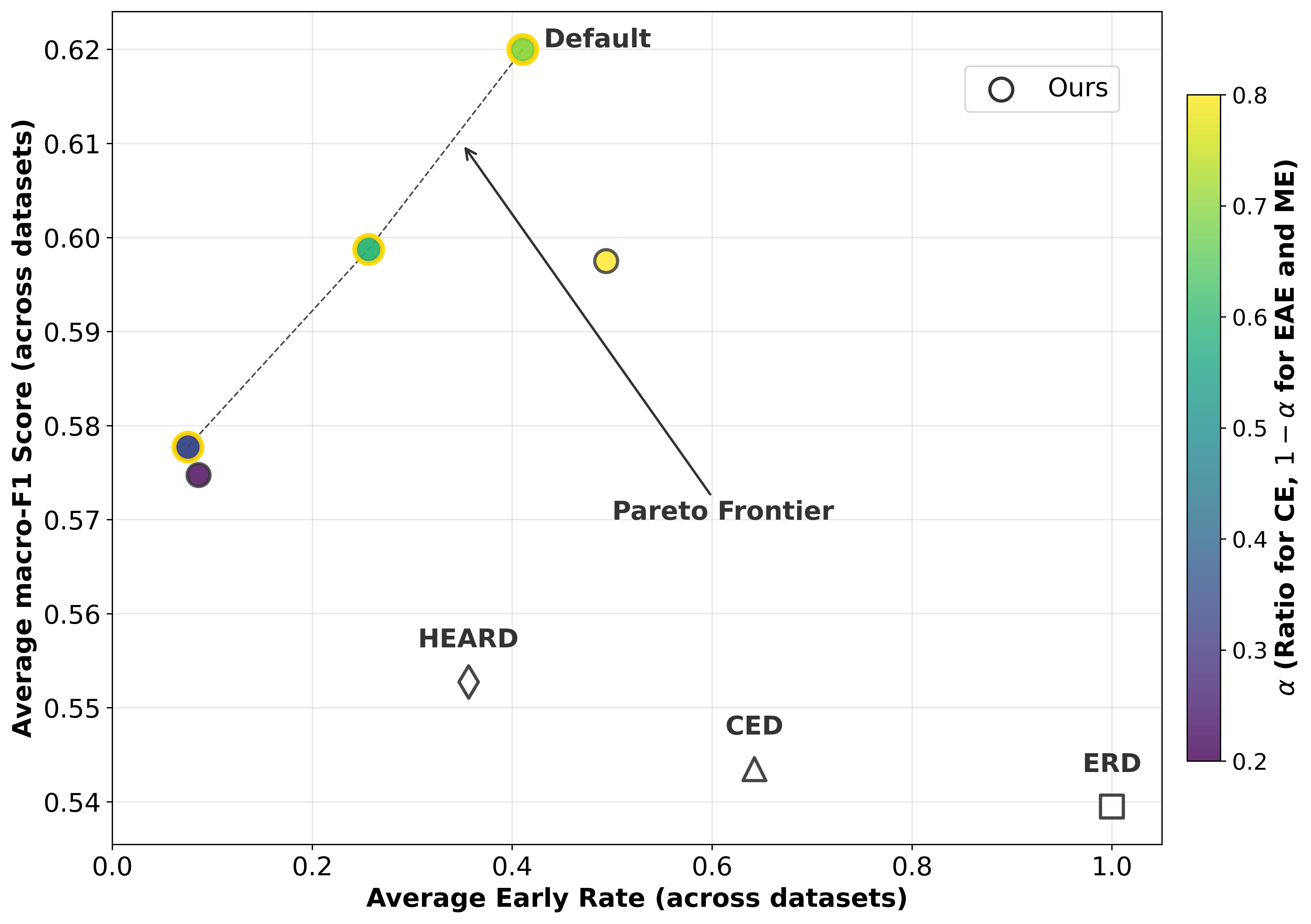}
  \caption{Trade-off between classification accuracy and earliness for our method (across different expert trajectory ratios) and existing EARD methods. Each marker shows average performance over 4 datasets. Our settings are circles, color-coded by $\alpha$ (CE). Gold outlines mark Pareto-optimal points, connected by the dashed black line to form the Pareto frontier.}
  \label{fig:different ratios}
\end{figure}

\subsubsection{Analysis on Trade-off between Classification Accuracy and Earliness.} \label{trade-off}
We analyze the trade-off between classification performance and ER by varying expert trajectory ratios in our method and comparing with EARD baselines. In Figure~\ref{fig:different ratios}, each circle represents a setting of our method, with a different color indicating the CE ratio ($\alpha$), while $1 - \alpha$ is equally distributed between EAE and ME. Baseline EARD methods are included for comparison. We identify those models as Pareto-optimal points, in which no other setting achieves both a higher F1 score and a lower ER. 
We observe that none of the EARD baselines lie on the Pareto frontier and all of them yield lower F1 scores compared to our method while CED/ERD even exhibit higher ER. In contrast, our method achieves superior trade-offs between accuracy and earliness, with several settings dominating the baselines and lie on the Pareto frontier.
We also observe that progressively decreasing the ratio of \textit{CE} (darker circle colors) while increasing the ratios of \textit{EAE} and \textit{ME} generally leads to a reduction in ER. This is expected, as a higher proportion of \textit{EAE} trajectories allows for earlier stopping compared to the more cautious \textit{CE} trajectories. However, this comes with a price, as it leads to lower F1 score. 
As discussed in \S\ref{Objective Function}, we prioritize the proportion of \textit{CE} trajectories and make \textit{EAE} and \textit{ME} trajectories contribute equally to strike a balance during training to enable the model to make stable and accurate early predictions without being overly cautious or aggressive. And among the tested settings, our default configuration achieves the highest F1 score and competitive ER on the Pareto frontier. These findings suggest the importance of balancing different expert trajectory types to optimize the agent’s policy for improving EARD.

\section{Conclusion}
We introduce a novel framework that combines LLMs with an autonomous neural agent to provide an agile and cost-effective solution for few-shot early rumor detection given data-limited scenarios. Our approach involves training the agent using an imitation learning algorithm, which leverages three types of expert trajectories carefully curated considering earliness, accuracy, and stability in early rumor detection task. Theoretically, our method is proven to yield an optimal policy. Experiments on four real-world datasets show our approach boosts performance across LLMs and surpasses existing EARD methods in accuracy and earliness.

\begin{acks}
This research/project is supported by the National Research Foundation, Singapore under its AI Singapore Programme (AISG Award No: AISG3-RP-2024-035).
\end{acks}

\clearpage


\bibliographystyle{ACM-Reference-Format}

\bibliography{sample-base}

\appendix

\section{Theoretical Analysis} \label{Proof of objective func}
The objective function can be represented as: 
\begin{equation}
\label{ILobj}
\begin{aligned}
 & \min_{\pi} -H(\pi) + \alpha \psi^*(\rho_\pi - \rho_{\pi_{c}} ) + \beta \psi^*(\rho_\pi - \pi_{e}) - \gamma \psi^*(\rho_\pi - \rho_{\pi_{m}} ), \\
& \text{where } \psi^*( \rho_\pi -\rho_{\pi_E}) = \max_{D} \mathbb{E}_\pi[\log D(s, a)] + \mathbb{E}_{\pi_E}[\log(1 - D(s, a))], \\
& \text{and } E = \{c, e, m\}.
\end{aligned}
\end{equation}

\noindent The proof involves showing that the optimal policy in an imitation learning setting, typically obtained by first solving the Inverse Reinforcement Learning (IRL) problem to derive the optimal reward function $r^*$ and then applying a Reinforcement Learning (RL) algorithm, can be compressed into optimizing a $\psi$-regularized objective. Our method is based on generative adversarial imitation learning~\cite{ho2016generative}, extending it to handle different types of expert trajectories in the context of LLM.

In the aforementioned equation, the objective is to find the saddle point of the minimax problem. Since the reward function $r(s, a)$ is unknown, our goal is to determine the optimal policy by utilizing the expert policies $\pi_{c}, \pi_{e},\pi_{m}$. To accomplish this, we utilize the maximum casual entropy IRL method \cite{ziebart2010modeling,ziebart2008maximum} to solve the following optimization problem:
\begin{equation}
\label{IRL}
\begin{aligned}
\max \limits_{r \in \mathcal{R} \atop} \left( \min \limits_{\pi \in \Pi} -H(\pi) - (\alpha+ \beta - \gamma ) \mathbb{E}_{\pi}[r(s,a)] \right) \\
\hspace{0.1in} + \alpha \mathbb{E}_{\pi_{c}}[r(s,a)] + \beta \mathbb{E}_{\pi_{e}}[r(s,a)] - \gamma \mathbb{E}_{\pi_{m}}[r(s,a)] \\
\end{aligned}
\end{equation}
where $H(\pi) \triangleq \mathbb{E}_\pi[-\log \pi(a|s)]$ is the $\bar{\gamma}$-discounted causal entropy of the policy $\pi$. In practice, $\pi_{c}, \pi_{e},\pi_{m}$ will only be provided as a set of trajectories sampled by executing $\pi_E$ in the environment, so the expected reward of $\pi_E$ in Eq (\ref{IRL}) is estimated using these sampled trajectories. Let $\mathcal{R}$ be a set of reward functions. Maximum casual entropy IRL aims to find a reward function $r\in \mathcal{R}$ that gives low rewards to the learner's policy while giving high rewards to the $\pi_{c}$ and $\pi_{e}$, and give low rewards to the learner's policy and the $\pi_{m}$. The optimal policy can be found via a reinforcement learning procedure:

\begin{equation}
\label{RL}
RL(r) = \mathop{\arg \min} \limits_{\pi \in \Pi} -H(\pi) - (\alpha+ \beta - \gamma )  \mathbb{E}_{\pi}[r(s,a)]
\end{equation}

We study the policies obtained through RL, utilizing rewards learned through IRL on the most extensive set of reward functions, denoted as $\mathcal{R}$ in Eq.\eqref{IRL}, which encompasses all functions mapping from $\mathbb{R}^{\mathcal{S} \times \mathcal{A}}$ to $\mathbb{R}$. However, since using a large $\mathcal{R}$ can lead to overfitting in the IRL process, we employ a convex reward function regularizer \cite{finn2016guided}, denoted as $\psi$, to define the IRL procedure:

\begin{equation}
\label{IRL_regu}
\begin{aligned}
& IRL_\psi(\pi_{c},\pi_{e}, \pi_{m} ) = \\
& \mathop{\arg\max}\limits_{r \in \mathbb{R}^{\mathcal{S} \times \mathcal{A}}} \; - \psi(r) + \Bigg( \min \limits_{\pi \in \Pi} -H(\pi)  
 - (\alpha+ \beta - \gamma )  \mathbb{E}_{\pi}[r(s,a)] \Bigg) \\
&  + \alpha \mathbb{E}_{\pi_{c}}[r(s,a)] + \beta \mathbb{E}_{\pi_{e}}[r(s,a)] - \gamma \mathbb{E}_{\pi_{m}}[r(s,a)] \\
\end{aligned}
\end{equation}

Given $\tilde{r} \in IRL_\psi(\pi_{c}, \pi_{e}, \pi_{m})$, our objective is to learn a policy through $RL(\tilde{r})$. To effectively characterize $RL(\tilde{r})$, it is often advantageous to transform optimization problems over policies into convex formulations. This is achieved by utilizing the occupancy measure $\rho_\pi$. For a policy $\pi \in \Pi$, its occupancy measure $\rho_\pi : \mathcal{S} \times \mathcal{A} \to \mathbb{R}$ is defined as: $\rho_\pi(s, a) = \pi(a|s) \sum_{t=0}^{\infty} \bar{\gamma}^t \mathcal{P}(s_t = s|\pi) $.
The occupancy measure represents the distribution of state-action pairs that an agent encounters while following policy $\pi$. This correspondence allows us to express the expected reward as:
$\mathbb{E}_\pi[r(s,a)] = \sum_{s,a} \rho_\pi(s,a) r(s,a)$ for any reward function $r$. Notably, there is a one-to-one correspondence between the set of policies, $\Pi$, and the set of occupancy measures, $\Gamma$, (Proposition \ref{proposition 1}).

\begin{proposition}
\label{proposition 1}
(Theorem 2 of \cite{syed2008apprenticeship}) If $\rho \in \Gamma$, then $\rho$ is the occupancy measure for $\pi_\rho(a|s) \triangleq \rho(s,a)/\sum_a' \rho(s,a')$, and $\pi_\rho$ is the only policy whose occupancy measure is $\rho$.
\end{proposition}

\begin{proposition}
\label{proposition 2}
(Lemma 3.1 of \cite{ho2016generative}) Let 
\[
\bar{H}(\rho) = -\sum_{s,a}\rho(s,a) \log\left(\frac{\rho(s,a)}{\sum_{a'} \rho(s,a')}\right).
\]
Then, $\bar{H}$ is strictly concave, and for all $\pi \in \Pi$ and $\rho \in \Gamma$, we have $H(\pi)=\bar{H}(\rho_\pi)$ and $\bar{H}(\rho)=H(\pi_\rho)$.
\end{proposition}

\begin{proposition}
\label{proposition 3}
Let $\tilde{r}  \in IRL_\psi(\pi_{c},\pi_{e}, \pi_{m})$, $\tilde{\pi} \in RL(\tilde{r}) = RL\circ IRL_{\psi_\phi}(\pi_{c},\pi_{e},\pi_{m})$, and 
\begin{equation}
\begin{aligned}
&\pi_A \in \mathop{\arg\min} \limits_{ \pi}  -H(\pi) + \alpha \psi^*(\rho_\pi -\rho_{ \pi_{c}}) + \beta \psi^*( \rho_\pi - \rho_{\pi_{e}}) \\
& - \gamma \psi^*( \rho_\pi -\rho_{\pi_{m}}) \\
&= \mathop{\arg\min} \limits_{ \pi} \max \limits_{r} -H(\pi) -\psi(r) + \\
& \sum_{s,a}[ \alpha ( \rho_\pi(s,a) - \rho_{\pi_{c}}(s,a) ) + \beta ( \rho_\pi(s,a) - \rho_{\pi_{e}}(s,a) )) \\
& - \gamma ( \rho_\pi(s,a) - \rho_{\pi_{m}}(s,a)) ] r(s,a) \\
\end{aligned}
\end{equation}
Then $\pi_A = \tilde{\pi}$.
\end{proposition}
\begin{proof}
Let $\rho_A$ be the occupancy measure of $\pi_A$ and $\tilde{\rho}$ be the occupancy measure of $\tilde{\pi}$. We define $\bar{L}:\mathcal{D}\times \mathbb{R}^{\mathcal{S}\times \mathcal{A}} \to \mathbb{R}$ by
\begin{equation}
\label{L}
\begin{aligned}
& \bar{L}(\rho,r) = -\bar{H}(\rho) - \psi(r) + \\
& \sum_{s,a}[ \alpha ( \rho(s,a)- \rho_{\pi_{c}}(s,a) ) + \beta ( \rho(s,a) - \rho_{\pi_{e}}(s,a) ) \\
& - \gamma ( \rho(s,a) - \rho_{\pi_{m}}(s,a) ) ] r(s,a) \\
\end{aligned}
\end{equation}

The following relationship then holds:
\begin{equation}
\label{P2}
\rho_A \in \mathop{\arg\min} \limits_{\rho \in \Gamma} \max \limits_{r} \bar{L}(\rho,r)
\end{equation}

\begin{equation}
\label{P3}
\tilde{r} \in \mathop{\arg\max} \limits_{r} \min \limits_{\rho \in \Gamma } \bar{L}(\rho,r)
\end{equation}

\begin{equation}
\label{P4}
\tilde{\rho} \in \mathop{\arg\min} \limits_{\rho \in \Gamma} \bar{L}(\rho,\tilde{r})
\end{equation}
where $\Gamma$ is compact and convex, $\mathbb{R}^{\mathcal{S}\times \mathcal{A}}$ is convex. Due to convexity of $-\bar{H}$ and $\psi(r)$,it follows that $\bar{L}(\rho,\cdot)$ is convex for all $\rho$. $\bar{L}(\cdot,r)$ is concave for all $r$ (see proof in \ref{Other 1}).

Therefore, we can use minimax duality~\cite{millar1983minimax}:
\begin{equation}
\label{P5}
\min \limits_{\rho \in \Gamma} \max \limits_{r \in \mathcal{R}} \bar{L}(\rho,r) = \max \limits_{r \in \mathcal{R}} \min \limits_{\rho \in \Gamma} \bar{L}(\rho,r)
\end{equation}

Hence, from Eqs.(\ref{P2}) and (\ref{P3}), $(\rho_A,\tilde{r})$ is a saddle point of $\bar{L}$, which implies that:
\begin{equation}
\label{P6}
\rho_A \in \mathop{\arg\min} \limits_{\rho \in \Gamma} \bar{L}(\rho,\tilde{r})
\end{equation}
Because $\tilde{L}(\cdot,r)$ is strictly concave for all $r$, Eqs.(\ref{P4}) and (\ref{P6}) imply $\rho_A = \tilde{\rho} $. Since policies whose corresponding occupancy measure are unique(Proposition \ref{proposition 2}), finally we get $\pi_A = \tilde{\pi}$
\end{proof}

Proposition \ref{proposition 3} illustrates the process of IRL in finding the optimal reward function, represented by $r^*$. By utilizing the output of IRL, reinforcement learning can be executed to obtain the optimal policy, represented by $\pi^*$. And we prove that $\pi^*$ is the same as by directly solving the $\psi$-regularized imitation learning problem $\tilde{L}$. Furthermore, $\psi$-regularized imitation learning aims to identify a policy whose occupancy measure is similar to that of an expert, as measured by the convex function $\psi^*$. Subsequently, we deduce the form of $\psi^*$.

\citet{ho2016generative} present a cost regularizer, $\psi_{GA}$, that leads to an imitation learning algorithm, as outlined in Eq.(\ref{ILobj}), which aims to minimize the Jensen-Shannon divergence between the occupancy measures. Specifically, they convert a surrogate loss function, $\phi$, which is used for binary classification of state-action pairs drawn from the occupancy measures $\rho_\pi$ and $\rho_{\pi_E}$, into cost function regularizers $\phi$, such that $\phi^*(\rho_\pi-\rho_{\pi_E})$ represents the minimum expected risk, $R_\phi(\rho_\pi,\rho_{\pi_E})$, for the function $\phi$~\cite{ho2016generative}:

\begin{equation}
 \label{expected risk}
  R_\phi(\rho_\pi,\rho_{\pi_E}) = \sum_{s,a} \max \limits_{\gamma \in \mathbb{R}}  \rho_\pi(s,a)\phi(\gamma) + \rho_{\pi_E}(s,a) \phi(-\gamma).   
\end{equation}

\citet{ho2016generative} use the formula of surrogate loss function $\phi$: $\psi_\phi(c) = \sum_{\rho_{\pi_E}}g_\phi(c(s,a))$, where $g_\phi(x) = -x + \phi(-\phi^{-1}(-x))$, and $\phi$ is a strictly decreasing convex function (Appendix A.2 Proposition A.1 from~\cite{ho2016generative}). However, in our approach, we adopt reward function $r(s,a)$ instead of the cost function $c(s,a)$, thus we write in this form: $\psi_\phi(r) = \sum_{\rho_{\pi_E}}g_\phi(r(s,a))$, where $g_\phi(r(s,a)) = x + \phi (- \phi^{-1}(x))$, and $\psi_\phi(r)$ is also closed, proper and convex (See proof in \ref{Other 2}).

Then formulation of $\psi_\phi^*(\rho_\pi - \rho_{\pi_E})$ is represented as follow (see proof in \ref{Other 2}):

\begin{equation}
\label{regu 1}
\begin{aligned}
&\psi_\phi^*( \rho_\pi - \rho_{\pi_E})
= -R_\phi(\rho_\pi,\rho_{\pi_E})   \\
\end{aligned}
\end{equation}

By using the logistic loss $\phi(\gamma) = \log (1 + e^{-\gamma})$~\cite{ho2016generative}, we have $-R_\phi(\rho_\pi,\rho_{\pi_E}) = \max \limits_{ D  \in (0,1)^{\mathcal{S}\times \mathcal{A}}}  \sum_{s,a} \rho_\pi(s,a) \log D(s,a)+ \rho_{\pi_E}(s,a) \log(1- D(s,a)), E = \{c, e, m \}$. Therefore, we obtain the final form of objective function as in Eq.~\eqref{ILobj}.

\section{Other Proofs}
\subsection{Prove convexity and concavity of $\bar{L}$}
\label{Other 1}
$\bar{L}(\cdot, r)$ is convex for all $\rho$ and concave for all $r$.

\begin{proof}
    
We know that $\psi(r)$ and $-\bar{H}(\rho)$ are convex, suppose $\lambda \in[0,1]$. To simplify, we denote $r(s,a)$ as $r$, $ \rho(s,a) - \rho_{\pi_E}(s,a)$ as $\rho - \rho_{\pi_E}$ for $E = \{c, e, m \}$.
\begin{equation}
\label{L_concave}
\begin{aligned}
&\bar{L}(\cdot, \lambda r_1+(1-\lambda)r_2) = -\bar{H}(\rho) - \psi(\lambda r_1 +(1-\lambda)r_2) +  \\
& \sum_{s,a} \Big[\alpha( \rho - \rho_{\pi_{c}} ) + \beta( \rho - \rho_{\pi_{e}}) - \gamma( \rho - \rho_{\pi_{m}} )\Big](\lambda r_1 + (1-\lambda) r_2) \\
& \geq  -\lambda \bar{H}(\rho) - (1-\lambda) \bar{H}(\rho) - \lambda \psi(r_1) - (1-\lambda) \psi(r_2) + \\
&\lambda  \sum_{s,a} \Big[\alpha(\rho -\rho_{\pi_{c}}  ) + \beta( \rho -\rho_{\pi_{e}} ) - \gamma( \rho - \rho_{\pi_{m}} )\Big] r_1  + \\
& (1-\lambda) \sum_{s,a} \Big[\alpha( \rho - \rho_{\pi_{c}}) + \beta( \rho - \rho_{\pi_{e}}) - \gamma( \rho - \rho_{\pi_{m}} )\Big] r_2 \nonumber
\end{aligned}
\end{equation}

Therefore, $\bar{L}(\cdot,(\lambda r_1+(1-\lambda) r_2) \geq \lambda \bar{L}(\cdot, r_1) + (1- \lambda) \bar{L}(\cdot,r_2)$, $\bar{L}(\cdot,r $ is concave for all $r$. The proof for $\rho $ is similar as $-\bar{H}(\rho)$ is convex.
\end{proof}

\subsection{Proof of $\psi_\phi(r)$}

\begin{proposition}
Suppose $\phi : \mathbb{R} \rightarrow \mathbb{R}$ is a strictly decreasing convex function. Let $T$ be the range of $\phi$, and define $g_\phi : \mathbb{R} \rightarrow \mathbb{R}$ and $\psi_\phi : \mathcal{R} \times \mathcal{A} \rightarrow \mathbb{R}$ by:
\[
g_\phi(x) = 
\begin{cases}
    x + \phi (- \phi^{-1}(x)) & \text{if } x \in T \\
    +\infty & \text{otherwise}
\end{cases}
\]
\[
\psi_\phi(r) = 
\begin{cases}
    \sum_{s,a} \rho_{\pi_E}(s,a) g_\phi(r(s,a)) & \text{if } r(s,a) \in T \text{ for all } s, a \\
    +\infty & \text{otherwise}
\end{cases}
\]
Then, $\psi_\phi$ is closed, proper, and convex, and RL$\circ$IRL$_{\psi_\phi}(\pi_{c},\pi_{e},\pi_{m}) = \arg\min_{\pi} -H(\pi) - \alpha  R_\phi(\rho_\pi, \rho_{\pi_{c}}) - \beta R_\phi(\rho_\pi, \rho_{\pi_{e}} ) + \gamma  R_\phi(\rho_\pi, \rho_{\pi_{m}})$.
\end{proposition}

\begin{proof}
To verify the first claim, we need to check that $g_\phi(x) = x + \phi^{-1}(x)$ is closed, proper, and convex. Convexity follows from the fact that the mapping $x \mapsto \phi(-\phi^{-1}(x))$ is convex (see proof in \ref{Other 2}). Additionally, since $T$ is nonempty, $g_\phi$ is proper.

To show that $g_\phi$ is closed, observe that since $\phi$ is strictly decreasing and convex, the range of $\phi$ is either $\mathbb{R}$ or an open interval $(b, \infty)$ for some $b \in \mathbb{R}$. If the range of $\phi$ is $\mathbb{R}$, then $g_\phi$ is finite everywhere and hence closed. 

If the range of $\phi$ is $(b, \infty)$, then $\phi(x) \rightarrow b$ as $x \rightarrow \infty$ and $\phi(x) \rightarrow \infty$ as $x \rightarrow -\infty$. As a result, as $x \rightarrow b$, $\phi^{-1}(x) \rightarrow \infty$, and therefore $\phi(-\phi^{-1}(x)) \rightarrow \infty$. This implies that $g_\phi(x) \rightarrow \infty$ as $x \rightarrow b$, ensuring that $g_\phi$ is closed.

\end{proof}

\subsection{Prove $\phi(-\phi^{-1}(x))$ is convex}
\label{Other 2}
\begin{proof}

Since $\phi : \mathbb{R} \rightarrow \mathbb{R}$ is a strictly decreasing convex function, its inverse, $\phi^{-1}$, is also convex. For $\lambda \in [0, 1]$, we have:
\[
- \phi^{-1}(\lambda x_1 + (1 - \lambda) x_2) \geq - \lambda \phi^{-1}(x_1) + (1 - \lambda) \phi^{-1}(x_2).
\]
Given that $\phi$ is decreasing, it follows that:
\[
\phi(- \phi^{-1}(\lambda x_1 + (1 - \lambda) x_2)) \leq \phi(- \lambda \phi^{-1}(x_1) + (1 - \lambda) \phi^{-1}(x_2)).
\]
Furthermore, because $\phi$ is convex, we have:
\[
\phi(- \lambda \phi^{-1}(x_1) + (1 - \lambda) \phi^{-1}(x_2)) \leq \lambda \phi(- \phi^{-1}(x_1)) + (1 - \lambda) \phi(- \phi^{-1}(x_2)).
\]
Thus, combining these inequalities gives:
\[
\phi(- \phi^{-1}(\lambda x_1 + (1 - \lambda) x_2)) \leq \lambda \phi(- \phi^{-1}(x_1)) + (1 - \lambda) \phi(- \phi^{-1}(x_2)).
\]
Therefore, $\phi(-\phi^{-1}(x))$ is convex.
\end{proof}

\subsection{Proof of $\psi_\phi^*(\rho_\pi - \rho_{\pi_E} )$}
\label{Other 3}

Here we give the formulation of $\psi_\phi^*(\rho_\pi - \rho_{\pi_E})$ as shown in eq. (\ref{regu 1}), for E in $\{ c, e, m\}$

\begin{equation}
\label{regu}
\begin{aligned}
& \psi_\phi^*( \rho_\pi  - \rho_{\pi_E}) = \max_{r \in \mathcal{R}} \sum_{s,a} ( \rho_{\pi_E}(s, a) - \rho_\pi(s, a)) r(s, a) \\
& - \sum_{s,a} \rho_{\pi_E}(s, a) g_\phi(r(s, a)) \\
& = \sum_{s,a} \max_{r \in T} (\rho_{\pi_E}(s, a) - \rho_\pi(s, a)) r - \rho_{\pi_E}(s, a) [r + \phi(-\phi^{-1}(r))] \\
& = \sum_{s,a} \max_{r \in T} -\rho_\pi(s, a)r - \rho_{\pi_E}(s, a) \phi(-\phi^{-1}(r))\\
& = \sum_{s,a} \max_{\gamma \in \mathbb{R}} -\rho_\pi(s, a) \phi(\gamma) - \rho_{\pi_E}(s, a) \phi(-\phi^{-1}(\phi(\gamma))) \\
& = \sum_{s,a} \max_{\gamma \in \mathbb{R}} \rho_\pi(s, a) (-\phi(\gamma)) - \rho_{\pi_E}(s, a) \phi(-\gamma) \\
& = -R_\phi(\rho_\pi, \rho_{\pi_E}) \\
\end{aligned}
\end{equation}
where we made the change of variables $r \rightarrow \phi(\gamma)$, justified because $T$ is the range of $\phi$. \qed

\section{Experiments}

\begin{figure*}
    \vspace{-0ex}
    \includegraphics[width=\textwidth]{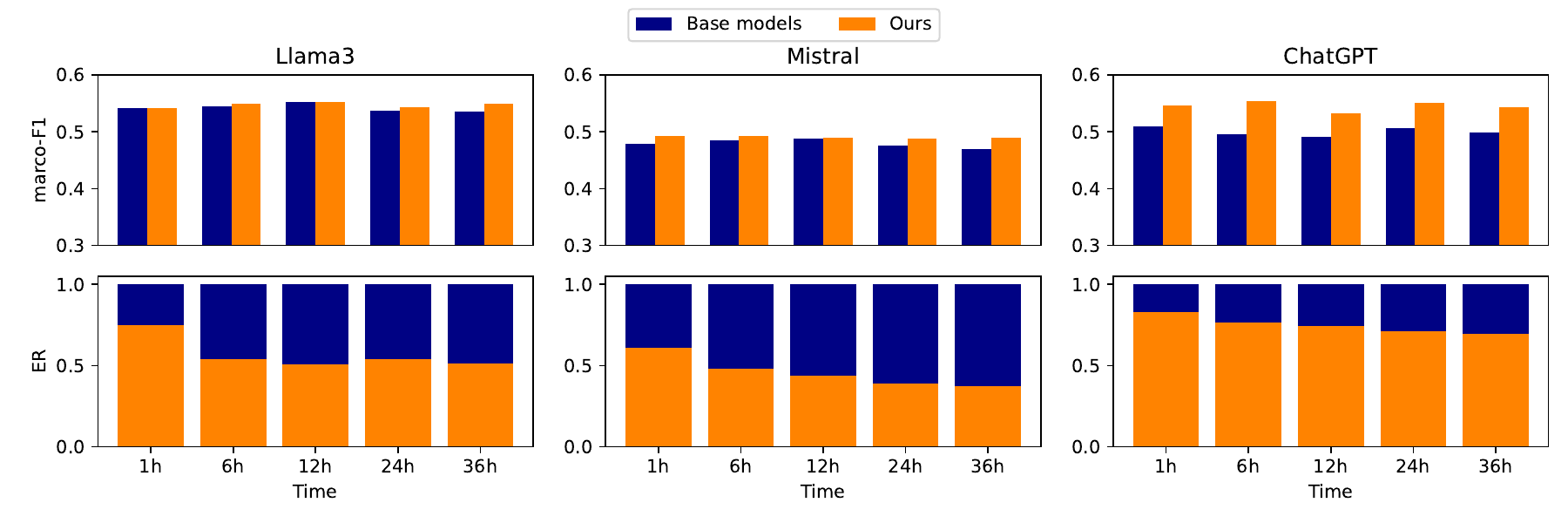}
    \vspace{-5ex}
\caption{Comparison of base models and ours on the EARD-oriented dataset BEARD. The x-axis represents different time intervals of posts used for prediction. The top row shows macro-F1 score, while the bottom row shows ER.}
    \label{fig: cttoff across E}
\end{figure*}

\begin{table*}[h]
\centering

\begin{tabular}{ccc|cc|cc|cc}
\toprule[1.0pt]
\multirow{2}{*}{\textbf{Shots}} 
& \multicolumn{2}{c|}{\textbf{PHEME}} 
& \multicolumn{2}{c|}{\textbf{Twitter}} 
& \multicolumn{2}{c|}{\textbf{BEARD}} 
& \multicolumn{2}{c}{\textbf{Twitter-COVID-19}}\\
\cline{2-9}
& \textbf{macro-F1} & \textbf{ER} & \textbf{macro-F1} & \textbf{ER} & \textbf{macro-F1} & \textbf{ER} & \textbf{macro-F1} & \textbf{ER} \\
\midrule[0.5pt]
$10$ 
& $0.658$\textsubscript{0.019} & $0.244$\textsubscript{0.079} 
& $0.675$\textsubscript{0.052} & $0.387$\textsubscript{0.433} 
& $0.526$\textsubscript{0.020} & $0.263$\textsubscript{0.258} 
& $0.586$\textsubscript{0.036} & $0.556$\textsubscript{0.362} \\
\midrule[0.5pt]
$20$ 
& $0.653$\textsubscript{0.012} & $0.191$\textsubscript{0.038} 
& $0.683$\textsubscript{0.033} & $0.283$\textsubscript{0.316} 
& $0.545$\textsubscript{0.041} & $0.437$\textsubscript{0.272} 
& $0.594$\textsubscript{0.033} & $0.406$\textsubscript{0.361} \\
\midrule[0.5pt]
$30$ 
& $0.653$\textsubscript{0.013} & $0.189$\textsubscript{0.037} 
& $0.670$\textsubscript{0.038} & $0.394$\textsubscript{0.323} 
& $0.529$\textsubscript{0.031} & $0.257$\textsubscript{0.263} 
& $0.583$\textsubscript{0.016} & $0.625$\textsubscript{0.292} \\
\midrule[0.5pt]
$40$ 
& $0.658$\textsubscript{0.016} & $0.347$\textsubscript{0.278} 
& $0.661$\textsubscript{0.042} & $0.230$\textsubscript{0.351} 
& $0.537$\textsubscript{0.045} & $0.413$\textsubscript{0.299} 
& $0.601$\textsubscript{0.014} & $0.573$\textsubscript{0.334} \\
\bottomrule[1.0pt]
\end{tabular}
\caption{Evaluation results of ChatGPT integrated with our framework under varying numbers of training shots across all datasets.}
\label{exp:shots}
\end{table*}

\subsection{Additional Results} 

\subsubsection{Comparison with Base Models} \label{appendix:Preset Time Checkpoints}
We show the comparison with various LLMs using Preset Time Checkpoints Strategy on the BEARD dataset, which is tailored for EARD task by including as many early-stage posts as possible. As shown in Figure \ref{fig: cttoff across E}, our method consistently outperforms various LLMs at all time checkpoints and achieves a much lower ER. This confirms the advantage of our agent's automatic early time point determination over relying on preset checkpoints that can lead to delayed detection and suboptimal performance.

\subsection{Results under varying numbers of training shots.} \label{appendix:shots}
Table~\ref{exp:shots} presents detailed results for each dataset under different numbers of training shots.

\end{document}